\documentclass[twoside]{article}

\usepackage[preprint]{aistats2026}

\usepackage[T1]{fontenc}
\usepackage{lmodern}
\usepackage{microtype}

\usepackage{amsmath,amssymb,amsfonts,mathtools}
\usepackage{amsthm}

\theoremstyle{plain}
\newtheorem{proposition}{Proposition}

\theoremstyle{definition}

\theoremstyle{remark}
\newtheorem{remark}{Remark}

\usepackage{bm}
\usepackage{bbm}
\usepackage{graphicx}
\usepackage{booktabs}
\usepackage{multirow}
\usepackage{array}
\usepackage{siunitx}
\usepackage{comment}
\usepackage{algorithm}
\usepackage{algorithmic}

\usepackage[round]{natbib}

\usepackage{xcolor}
\usepackage[hidelinks]{hyperref}
\hypersetup{
  colorlinks=true,
  linkcolor=blue,
  citecolor=teal,
  urlcolor=blue
}
\newcommand{\R}{\mathbb{R}}
\newcommand{\E}{\mathbb{E}}
\newcommand{\Cov}{\mathrm{Cov}}
\newcommand{\tr}{\mathrm{tr}}
\newcommand{\diag}{\mathrm{diag}}
\newcommand{\Io}{\bm{\mathcal{I}}_{\mathrm{obs}}}
\newcommand{\Ic}{\bm{\mathcal{I}}_{\mathrm{comp}}}
\newcommand{\Imiss}{\bm{\mathcal{I}}_{\mathrm{miss}}}
\newcommand{\Iohat}{\bm{\widehat{\mathcal{I}}}_{\mathrm{obs}}}
\newcommand{\Iostephat}{\bm{\widehat{\mathcal{I}}}_{\mathrm{obs,step}}}

\hypersetup{pdfauthor={Andre Herz, Georgia Koppe, Daniel Durstewitz}}

\begin{document}

\runningtitle{Teacher Forcing as Generalized Bayes}

\twocolumn[
    \aistatstitle{Teacher Forcing as Generalized Bayes: Optimization Geometry Mismatch in Switching Surrogates for Chaotic Dynamics}
    \vspace{-2em}
    \aistatsauthor{Andre Herz$^{1,*}$ \And
    Daniel Durstewitz$^{1,2}$ \And
    Georgia Koppe$^{1,3,4}$ 
    }
    \vspace{1em}
    \aistatsaddress{
    $^{1}$Interdisciplinary Center for Scientific Computing, Heidelberg University, Germany \\
    $^{2}$Dept. of Theoretical Neuroscience, Central Institute of Mental Health (CIMH), Mannheim, Germany\\
    $^{3}$Hector Institute for AI in Psychiatry and Dept. of Psychiatry and Psychotherapy, CIMH, Mannheim, Germany\\
    $^{4}$Hertie Institute for AI in Brain Health, University of Tübingen, Germany\\
    $^{*}$Correspondence to: \texttt{andre.herz@iwr.uni-heidelberg.de}
    }
]

\begingroup
\renewcommand\thefootnote{}
\footnotetext{This work was presented at the Workshop on Optimisation and Post-Bayesian Inference in Machine Learning at AISTATS 2026.}
\endgroup
\vspace{-0.8em}

\begin{abstract}
\vspace{-0.5em}
Identity teacher forcing (ITF) enables stable training of deterministic recurrent surrogates for chaotic dynamical systems and has been highly effective for dynamical systems reconstruction (DSR) with recurrent neural networks (RNNs), including interpretable almost-linear RNNs (AL-RNNs). However, as an intervention-based prediction loss (and thus a generalized Bayes update), teacher forcing need not match the free-running model’s marginal likelihood geometry. We compare the objective-induced curvatures of ITF and marginal likelihood in a probabilistic switching augmentation of AL-RNNs, estimating ambiguity-aware observed information via Louis’ identity. In the switching setting studied here, conditioning on a single forced regime path (as ITF does) inflates curvature, while marginal likelihood curvature is reduced by a missing-information correction when multiple switching explanations remain plausible. In Lorenz--63 experiments, windowed evidence fine-tuning improves held-out evidence but can degrade dynamical quantities of interest (QoIs) relative to ITF-pretrained models.
\end{abstract}
\vspace{-0.5em}
\section{Introduction}
\vspace{-0.5em}
In dynamical systems reconstruction (DSR), one aims to learn surrogate models of unknown dynamical systems from time series data. The central goal is to recover the long-term behavior of the underlying system, beyond short-horizon prediction, commonly assessed via dynamical invariants and statistics such as Lyapunov exponents (LEs), the physical invariant measure governing time-averages, and topological and geometric properties of the attractor \citep{durstewitz2023reconstructing,goering2024oodg}. 

In practice, such models are typically trained using short-horizon prediction objectives. In chaotic systems, however, short-horizon predictive fit can be statistically decoupled from long-run invariant structure \citep{park2024statisticallyaccurate}. Positive LEs cause small state errors to grow exponentially, rendering trajectory-level gradients propagated through backpropagation through time (BPTT) ill-conditioned due to exploding gradients \citep{mikhaeil2022chaosrnn}. As a result, direct long-rollout training of recurrent neural network (RNN) surrogates is numerically unstable. This creates a training paradox: we seek long-run invariant behavior, yet optimization breaks down at long horizons.

A common workaround relies on sparse teacher forcing (STF) \citep{mikhaeil2022chaosrnn}. In STF, the latent state of a deterministic model is periodically fed a control signal derived from observations, anchoring rollouts to the data and thereby stabilizing gradient propagation \citep{mikhaeil2022chaosrnn}. This enables effective learning even when the autonomous (free-running) model would quickly diverge. Empirically, models trained under this principle have demonstrated striking success in reconstructing chaotic attractors, including almost-linear RNNs (AL-RNNs) with minimal nonlinear structure \citep{hess2023gtf,brenner2024alrnn}.

Nevertheless, STF optimizes prediction error under an explicit intervention that alters the rollout distribution, meaning it generally does not correspond to maximum-likelihood training of the unforced generative model. From a statistical viewpoint, STF instead defines a task loss that can be interpreted as a generalized Bayes (Gibbs) posterior, inducing its own local geometry in parameter space \citep{catoni2007pacbayes, bissiri2016generalbayes}. Under this generalized framework, this local geometry is fundamental: it underpins second-order optimization, Laplace uncertainty quantification, local identifiability analysis, and optimal experimental design.

By contrast, stochastic state-space models (SSMs) start from an explicit autonomous generative model and define a marginal likelihood by integrating over latent variables \citep{durbin2001ssm}. The local geometry of the resulting posterior is summarized by the observed information. For latent-variable SSMs, this can be computed via Louis' identity \citep{louis1982}.
As a modeling stance consistent with representing probabilistic beliefs, stochastic SSMs provide a principled default for real-world time series, since partial observability, measurement noise, and unmodeled perturbations are ubiquitous. Yet deterministic STF still dominates many chaotic DSR benchmarks \citep{hess2023gtf}.

We compare the local geometries induced by STF and marginal likelihood-based learning via ambiguity-aware observed information in a probabilistic switching SSM. AL-RNNs provide an interpretable testbed with explicit switching codes and tractable (closed-form) Fisher-information expressions aligned with ITF training. We then test whether windowed marginal likelihood optimization aligns with long-horizon dynamical invariants.
\paragraph{Related Work.}
Chaotic DSR surrogates are commonly trained with sparse/generalized teacher forcing to stabilize BPTT \citep{mikhaeil2022chaosrnn,hess2023gtf,brenner2024alrnn}. Probabilistic switching SSMs (e.g., SLDS/rSLDS) instead learn by marginal likelihood objectives with approximate inference/EM/SMC \citep{ghahramani2000switchingssm,murphy1998switchingkf,linderman2017rslds}. We connect these views by (i) treating STF as a generalized Bayes update \citep{catoni2007pacbayes,bissiri2016generalbayes} and (ii) estimating ambiguity-aware likelihood curvature via Louis’ identity \citep{louis1982}.
\vspace{-0.5em}
\section{Two Posteriors and Local Curvatures}
\vspace{-0.5em}
\label{sec:twoposteriors}
\subsection{AL-RNNs and Switching Dynamics}
AL-RNNs combine linear recurrent structure with a small number of gated nonlinear units, resulting in dynamics that switch between linear regimes \citep{brenner2024alrnn}. Let $\bm{z}_t\in\R^{M}$ be the latent state at time $t$. The AL-RNN transition model is
\begin{equation}
\label{eq:alrnn}
\bm{z}_{t+1}
= F_{\bm\theta}(\bm{z}_t)
:= \bm{A} \bm{z}_t + \bm{W}\,\bm\phi^\ast(\bm{z}_t) + \bm{h},
\end{equation}
where $\bm{A}, \bm{W} \in \mathbb{R}^{M \times M}$ are diagonal and full, respectively, and $\bm{h} \in \mathbb{R}^M$ is a bias term. $\bm\phi^\ast$ is identity on the first $M-P$ coordinates and ReLU on the last $P$. The architecture induces \emph{switching dynamics}. Define the binary switching code $\bm c_t\in\{0,1\}^P$ by $c_{t,j}=\mathbbm{1}\{z_{t,M-P+j}>0\}$, which specifies the active linear regime at time $t$. Let $\bm{D}(\bm{c}_t)$ denote the diagonal matrix that applies the corresponding gates:
\begin{equation}
\label{eq:Dc}
\bm{D}(\bm{c}_t)
=\diag\big( \underbrace{1,\dots,1}_{M-P},\,c_{t,1},\dots,c_{t,P} \big).
\end{equation}
Then the dynamics are piecewise affine and can be written as
\begin{equation}
\label{eq:pwa}
\bm{z}_{t+1}=\big(\bm{A}+\bm{W}\bm{D}(\bm{c}_t)\big)\bm{z}_t+\bm{h}.
\end{equation}
The induced switching code (gate) sequence $\{\bm {c}_t\}$ defines a symbolic partition of state space and determines which linear regime governs local evolution \citep{brenner2024alrnn}. Near switching boundaries, multiple gate configurations may be locally plausible, giving rise to latent uncertainty about which linear regime generated the data. We will empirically show that whether or not this uncertainty is properly accounted for has direct consequences for the local curvature of the training objective.

\subsection{ITF Loss Posterior and Local Curvature}
As training protocol, we use \emph{identity TF} (ITF), a variant of STF where the first $N$ latent coordinates (the observed subspace) are overwritten at forcing times $\mathcal T_\tau$ \citep{mikhaeil2022chaosrnn, brenner22dendritic}.
Let $\bm{B}=[\bm{I}_N\ \bm{0}]\in\R^{N\times M}$ denote the projection onto observed coordinates such that $\hat{\bm{x}}_t=\bm{B}\bm{z}_t$, and let $\mathcal T_\tau=\{t:t\equiv 0\ (\mathrm{mod}\ \tau),\,t>0\}$ denote the forcing times.
The ITF-modified rollout can be written schematically as
\begin{equation}
\label{eq:itf_rollout}
\tilde{\bm z}_t = \bm z_t + \bm B^{\mathsf T}\!\big(\bm x_t - \bm B\bm z_t\big),
\quad
\bm z_{t+1} =
\begin{cases}
F_{\bm\theta}(\tilde{\bm{z}}_t), & t\in\mathcal{T}_\tau,\\
F_{\bm\theta}(\bm{z}_t), & \text{else.}
\end{cases}
\end{equation}
STF can also be generalized using a forcing strength parameter to linearly interpolate between $\bm{z}_t$ and $\tilde{\bm{z}}_t$ \citep{hess2023gtf}. Here, we focus on ``strict'' forcing. Training based on BPTT then minimizes the one-step prediction loss along an ITF-modified trajectory:
\begin{equation}
\label{eq:tf_loss}
\mathcal{L}_{\mathrm{ITF}}(\bm\theta;\bm{x}_{1:T})
= \frac{1}{T-1}\sum_{t=1}^{T-1}\big\|\bm{B}\bm{z}_{t+1}-\bm{x}_{t+1}\big\|_2^2.
\end{equation}
This loss evaluates prediction error under an intervention that alters the system's dynamics during training. It therefore does not correspond to the maximum likelihood estimator for the autonomous generative model one uses at test time (i.e., free-running dynamics).
A useful framework for interpreting this objective statistically is generalized Bayes \citep{catoni2007pacbayes,bissiri2016generalbayes}. In chaotic DSR, global likelihoods are brittle, and any surrogate model is necessarily misspecified relative to the underlying physical system. Rather than assuming a single canonical likelihood, it is therefore natural to treat training objectives as scoring rules that induce generalized posteriors and, consequently, distinct local parameter geometries. Interpreting $\mathcal L_{\mathrm{ITF}}$ as such a scoring rule, ITF corresponds to the Gibbs posterior
\begin{equation}
\label{eq:gibbs}
\pi_{\mathrm{ITF}}(\bm\theta\mid\bm x)
\propto
\pi_0(\bm\theta)\exp\!\big(-\beta\,\mathcal L_{\mathrm{ITF}}(\bm\theta;\bm x)\big),
\end{equation}
where $\pi_0$ is a prior and $\beta>0$ is an inverse-temperature (loss-scale) parameter. We define the ITF curvature proxy as the generalized per-step Gauss–Newton/Fisher matrix aligned to ITF:
\begin{equation}
\label{eq:tfgn}
\bm{\mathcal{I}}_{\mathrm{ITF}}
=
\frac{1}{T-1}\sum_{t=1}^{T-1}
\bm{J}_t^{\mathsf T}\,\bm{\Lambda}^{-1}\,\bm{J}_t,
\quad
\bm{J}_t:=\frac{\partial(\bm{B}\bm{z}_{t+1})}{\partial\bm\theta},
\end{equation}
where $\bm{\Lambda}$ is a chosen loss weighting. 
At forcing times $\mathcal{T}_\tau$, the sensitivity recursion is modified to reflect the overwrite of the observed coordinates (Appendix~\ref{sec:s:tf_fim}). Crucially, $\bm J_t$ is evaluated along the forced rollout’s induced switching code sequence $\{\bm c_t\}$. Thus, the local geometry is conditioned on a single code path and does not marginalize over \emph{switching ambiguity} (alternative plausible regime paths).

All curvature diagnostics we report are data curvature only (no prior Hessian added). Since regularization is fixed across objectives, comparisons isolate objective-induced curvature.

\subsection{Marginal Likelihood Curvature and Observed Information via PAL-RNN}
To obtain a marginal likelihood-based reference geometry, we probabilistically augment the AL-RNN into a hybrid switching SSM (PAL-RNN). Instead of treating the switching codes $\bm{c}_t$ as deterministic functions of $\bm{z}_t$, we model them as latent random variables governed by probabilistic gating. Starting from a switching linear-Gaussian SSM:
\begin{align}
\label{eq:ssmtransition}
\bm{z}_{t+1}\mid \bm{z}_t,\bm{c}_t
&\sim
\mathcal{N}\!\Big(\big(\bm{A}+\bm{W}\bm{D}(\bm{c}_t)\big)\bm{z}_t+\bm{h},\ \bm{Q}\Big),\\
\bm{x}_t\mid \bm{z}_t
&\sim
\mathcal{N}\!\big(\bm{B}\bm{z}_t,\ \bm{R}\big),
\nonumber
\end{align}
we introduce probit gate noise:
\begin{equation}
\label{eq:probit}
c_{t,j}\mid \bm{z}_t
\sim
\mathrm{Bernoulli}\!\left(\Phi\!\left(\frac{z_{t,M-P+j}}{\sigma_g}\right)\right),
\end{equation}
where $\Phi$ is the standard normal CDF and $\sigma_g>0$ is a hyperparameter that controls gate stochasticity. $\bm Q$ and $\bm R$ are process/observation-noise covariances. In our experiments we use isotropic noise, i.e., $\bm Q=\sigma_{\mathrm{proc}}^2\bm I$ and $\bm R=\sigma_{\mathrm{obs}}^2\bm I$ on the relevant state/observation dimensions. We further interpret the squared error in \eqref{eq:tf_loss} as a Gaussian negative log-pseudo-likelihood with covariance $\bm R$, so $\bm R = \bm \Lambda$, the loss weight in the ITF curvature \eqref{eq:tfgn}.\\

Under the PAL-RNN formulation, the local curvature of the marginal likelihood $p_{\bm{\theta}}(\bm{x}_{1:T})$ of an observed time series is given by the \emph{observed information}, which, by Louis' identity \citep{louis1982}, decomposes into a complete-data term minus a missing-information correction arising from latent uncertainty:
\begin{align}
\Io(\bm\theta)
&=
\E\!\left[\Ic(\bm\theta)\mid \bm{x}\right]
\nonumber\\
&-
\underbrace{\Cov\!\left(
\nabla_{\bm\theta}\log p(\bm{x},\bm{z},\bm{c}\mid\bm\theta)
\,\middle|\,\bm{x}
\right)}_{\Imiss(\bm\theta)\succeq \bm{0}},
\label{eq:louis}
\end{align}
where $\Ic(\bm\theta):=-\nabla_{\bm\theta}^2 \log p_{\bm\theta}(\bm x,\bm z,\bm c)$ is the complete-data information. Missing information is large when multiple switching code explanations remain plausible \emph{a posteriori}. Under a Laplace approximation at $\hat{\bm\theta}$, $\Io(\hat{\bm\theta})$ provides the marginal likelihood contribution to posterior precision. Thus, objective curvature dictates local Gaussian uncertainty under both likelihood and loss-based frameworks. Next, we compare the ITF-induced pseudo-likelihood geometry with the marginal likelihood-based observed-information geometry and examine their implications for local identifiability and long-horizon DSR.
\vspace{-0.5em}
\section{Results}
\vspace{-0.5em}
\paragraph{Missing Information vs.\ Posterior Entropy}
To isolate the Louis mechanism in a setting where switching ambiguity can be varied without changing the underlying linear regimes, we study a probit-gated switching AR(1) toy model (Appendix~\ref{sec:s:louis_toy} and Appendix~\ref{sec:s:rbpf_louis_details}),
and sweep the gate-noise scale $\sigma_g$.
For each run, we relate the (time-averaged) mean posterior gate entropy to the missing-information ratio
($\mathrm{MIR}:=1-\mathrm{tr}(\Io)/\mathrm{tr}(\mathbb{E}[\Ic\mid \bm x_{1:T}])$), and the observed curvature proxy $\log_{10}\mathrm{tr}(\Io)$. Figure~\ref{fig:results}(a) shows that increasing entropy increases MIR and decreases observed curvature. This validates that posterior switching ambiguity flattens the marginal likelihood geometry through the Louis missing-information term. 

\begin{figure*}[t]
\centering
\includegraphics[width=\textwidth]{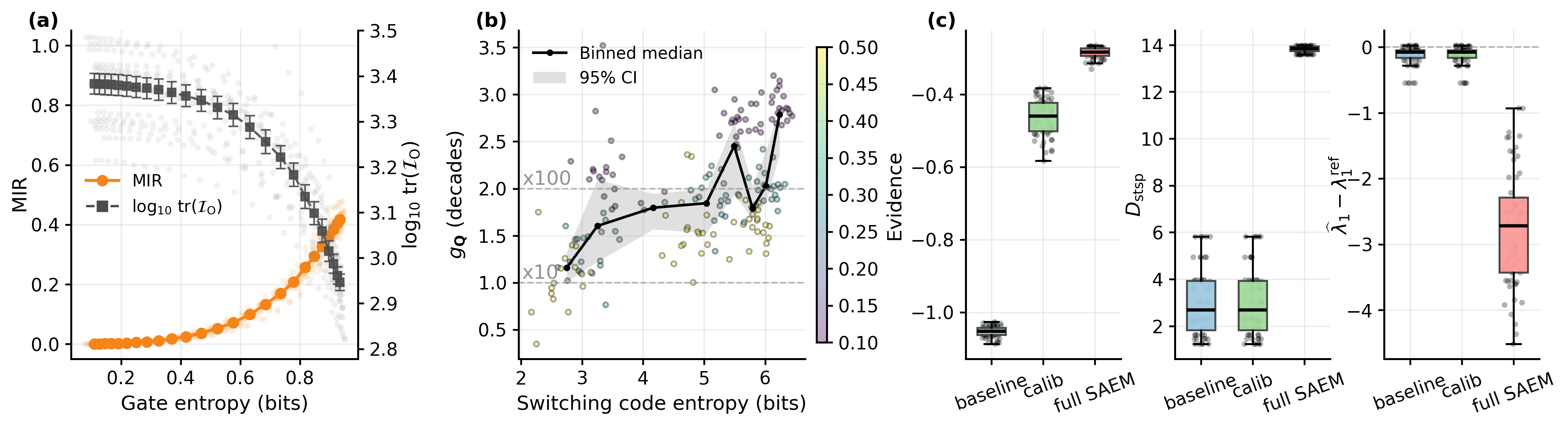}
\vspace{-1.2ex}
\caption{Summary of results.
\textbf{(a)} In a probit-gated switching AR(1) model, increasing gate noise
$\sigma_g$ leads to higher posterior switching ambiguity. This increases both mean posterior gate entropy and missing-information ratio (MIR), while decreasing observed-curvature proxy $\log_{10}\,\mathrm{tr}(\Io)$. Faint points $=$ individual runs; markers $=$ mean $\pm$ SEM across 20 seeds.
\textbf{(b)} Curvature gap $g_{\bm Q}$ (measuring how much ITF curvature exceeds ambiguity-aware observed information) vs.\ filtering switching code entropy $H_c$ evaluated at the same drift parameters. The gap indicates that conditioning on a single forced switching path inflates curvature under switching ambiguity. Points are colored by observation-noise level $\sigma_{\mathrm{obs}}$.
\textbf{(c)} Particle-SAEM fine-tuning: held-out windowed conditional log evidence (mean over windows; normalized per step and per observed dimension) and two hard-gated QoIs: $D_{\mathrm{stsp}}$ and signed LE error $\widehat{\lambda}_1-\lambda_1^{\mathrm{ref}}$ (negative $=$ overly contracting dynamics). \texttt{calib} updates process/observation noise only (QoIs unchanged); \texttt{full SAEM} updates drift+noise, improves evidence the most, but degrades both QoIs.}
\label{fig:results}
\end{figure*}

\paragraph{ITF-aligned Curvature Exceeds Ambiguity-Aware Observed Information.}
To examine whether ITF yields sharper local curvature than ambiguity-aware observed information under switching ambiguity, we compare ITF-aligned curvature to ambiguity-aware observed information. We test this on AL-RNNs trained under ITF on Lorenz--63 data \citep{lorenz1963deterministic}. For each model with drift parameters $\hat{\bm\theta}=\{\bm A,\bm W,\bm h\}$, we construct a probabilistic counterpart by applying the PAL-RNN formulation (\ref{eq:ssmtransition},~\ref{eq:probit}) and adding isotropic process/observation and probit gate noise (Sec.~\ref{sec:twoposteriors}). We compute the ITF-aligned curvature $\bm{\mathcal I}_{\mathrm{ITF}}(\hat{\bm\theta})$ along the forced rollout \eqref{eq:itf_rollout}. Using a Rao--Blackwellized particle smoother, we estimate $\Io(\hat{\bm\theta})$ via Louis' identity \eqref{eq:louis} while sweeping over noise levels (see Appendix~\ref{sec:s:rbpf_louis_details} for details). We quantify the mismatch by the local curvature gap
\begin{equation}
    \label{eq:gap}
    g_{\bm Q}=\log_{10}\!\left(T\,\mathrm{tr}(\bm{\mathcal I}_{\mathrm{ITF}})\,/\,\mathrm{tr}(\Io)\right),
\end{equation}
which estimates curvature inflation under ITF for a fixed process-noise scale $\bm Q$. Relating $g_{\bm Q}$ to (time-averaged) filtering switching code entropy $H_c$, Figure~\ref{fig:results}(b) shows gaps ranging from  $\approx 10$ up to $\approx 10^3$ that tend to increase with $H_c$. Appendix~\ref{sec:s:matrix_mismatch} shows that this curvature gap is not a uniform matrix dominance: most generalized-eigenvalue directions are not sharper under ITF, while trace summaries and leading-subspace overlaps still indicate an anisotropic, subspace-dependent geometry mismatch.

\paragraph{Evidence–QoI Misalignment}
To assess whether higher evidence implies better dynamics, we fine-tune PAL-RNNs initialized from ITF-trained AL-RNN checkpoints by maximizing windowed conditional log evidence with particle stochastic approximation expectation--maximization (particle-SAEM; E-step via a Rao--Blackwellized particle smoother; Appendix~\ref{sec:s:saem_details}). We compare (i) \texttt{baseline} (no fine-tuning), (ii) \texttt{calib} (process/observation noise only), and (iii) \texttt{full SAEM} (noise + drift). Across $12$ ITF-pretrained initializations and five window lengths
(detailed settings in Appendix~\ref{sec:s:datasets}), we evaluate held-out evidence and QoIs: overlap in state distribution $D_{\mathrm{stsp}}$ and largest LE error $\widehat{\lambda}_1-\lambda_1^{\mathrm{ref}}$ (QoI details in Appendix~\ref{sec:s:repro_and_metrics}). Hard-gated QoIs use deterministic rollouts with thresholded gates and zero noise.

Figure~\ref{fig:results}(c) reveals an evidence–QoI misalignment across models and window lengths. Although both \texttt{calib} and \texttt{full SAEM} increase held-out evidence (with \texttt{calib} leaving the hard-gated QoIs unchanged by construction), updating drift parameters (\texttt{full SAEM}) substantially degrades invariant reconstruction. In particular, it drives the signed largest Lyapunov error $\widehat{\lambda}_1-\lambda_1^{\mathrm{ref}}$
from near zero to strongly negative, indicating that windowed evidence trades chaos for stability while distorting long-run invariant structure. 
\vspace{-0.5em}
\section{Implications for Post-Bayesian DSR}
\vspace{-0.5em}
Local geometry matters in chaotic DSR: it drives local optimization, uncertainty quantification, identifiability, and experimental design. In switching surrogates, it is strongly objective-dependent: teacher-forced curvature conditions on a single forced switching path, whereas ambiguity-aware observed information flattens curvature when multiple regime explanations remain plausible. Windowed marginal likelihood fine-tuning can also be misaligned with long-horizon goals: evidence may improve while dynamical invariants degrade when drift parameters are updated. Together, these findings suggest there is no canonical posterior for chaotic DSR. The update and geometry should be chosen for the target QoI, motivating QoI-aware post-Bayesian learning and geometry-aware active data collection.

\subsubsection*{Acknowledgments}
This work was funded by the Baden-Württemberg Ministry of Science as part of the Excellence Strategy of the German Federal and State Governments, by the Federal Ministry of Research, Technology and Space (BMFTR) under the neuroAI initiative (01GQ2509B), by the German Research Foundation (Du 354/15-1), and the Hector II foundation.

\bibliographystyle{apalike}
\bibliography{refs_optimal}

\clearpage
\onecolumn

\renewcommand{\thesection}{A\arabic{section}}
\renewcommand{\thesubsection}{A\arabic{section}.\arabic{subsection}}
\renewcommand{\theequation}{A\arabic{equation}}
\renewcommand{\thealgorithm}{A\arabic{algorithm}}
\renewcommand{\theproposition}{A\arabic{proposition}}
\renewcommand{\theremark}{A\arabic{remark}}

\renewcommand{\theHsection}{A\arabic{section}}
\renewcommand{\theHsubsection}{A\arabic{section}.\arabic{subsection}}
\renewcommand{\theHequation}{A\arabic{equation}}
\providecommand{\theHalgorithm}{\arabic{algorithm}}
\renewcommand{\theHalgorithm}{A\arabic{algorithm}}
\providecommand{\theHproposition}{\arabic{proposition}}
\renewcommand{\theHproposition}{A\arabic{proposition}}
\providecommand{\theHremark}{\arabic{remark}}
\renewcommand{\theHremark}{A\arabic{remark}}

\setcounter{equation}{0}
\setcounter{algorithm}{0}
\setcounter{proposition}{0}
\setcounter{remark}{0}

\runningtitle{Appendix: Teacher Forcing as Generalized Bayes}

\aistatstitle{Appendix for\\
Teacher Forcing as Generalized Bayes: Optimization Geometry Mismatch in Switching Surrogates for Chaotic Dynamics}

\subsubsection*{Contents}
\noindent Appendix~\ref{sec:a:primer}: Dynamical systems reconstruction (DSR) \dotfill\ \pageref{sec:a:primer}\\
Appendix~\ref{sec:s:method_details}: Methodological details \dotfill\ \pageref{sec:s:method_details}\\
Appendix~\ref{sec:s:datasets}: Datasets and experimental settings \dotfill\ \pageref{sec:s:datasets}\\
Appendix~\ref{sec:s:repro_and_metrics}: Metric definitions and computation details \dotfill\ \pageref{sec:s:repro_and_metrics}\\
Appendix~\ref{sec:s:optional}: Additional analyses \dotfill\ \pageref{sec:s:optional}\\
Appendix~\ref{sec:s:limitations}: Limitations and future work \dotfill\ \pageref{sec:s:limitations}

\setcounter{section}{-1}
\section{Dynamical systems reconstruction (DSR)}
\label{sec:a:primer}
This section provides additional details on the field of DSR, its dynamical quantities of interest (QoIs), and the challenges of training surrogates for chaotic systems. It also introduces the AL-RNN surrogate model used in this work and motivates why objective geometry matters in DSR.
\subsection{DSR overview}
The goal of DSR is to learn a 
generative surrogate model of an unknown time-evolving process from observed time series. Unlike conventional forecasting, the scientific target is typically not only to remain aligned with one particular trajectory, but to recover the underlying dynamical mechanisms well enough to reproduce long-run behavior and support downstream analysis via interpretable model architectures.
\subsubsection{A minimal formalization}
Throughout, we focus on discrete time, $t \in \mathbb{Z}$. A deterministic DS is then expressed by an iterated map
\begin{equation}
\label{eq:a0:ds_map}
\bm{z}_{t+1}=F^{\star}(\bm{z}_t),\qquad \bm{z}_t\in \mathcal{Z} \subseteq\mathbb{R}^M,
\end{equation}
where $\mathcal{Z}$ is the (compact, measurable) \emph{state space} of the system and observations are generated through an (often non-invertible) measurement operator $g$,
\begin{equation}
\label{eq:a0:obs}
\bm{x}_t = g(\bm{z}_t) + \bm{\eta}_t,
\end{equation}
with observation noise $\bm{\eta}_t$. DSR trains a generative model to approximate a surrogate map $F_{\bm\theta}$ (and possibly a surrogate observation operator) so that the induced dynamics match the true system in an appropriate DS sense.
An idealized target is topological conjugacy: there exists a homeomorphism $h$ such that $h(F^{\star}(\bm{z}))=F_{\bm\theta}(h(\bm{z}))$.
In practice this is often too strong globally. Instead, one typically aims for agreement restricted to the \emph{attractor} (and its basin of attraction) that is supported by the observed data, and assesses it through dynamical QoIs \citep{durstewitz2023reconstructing,goering2024oodg}. An attractor is a forward-invariant set that attracts a neighborhood of initial conditions in state space. In applications, we only observe trajectories drawn from (or quickly converging to) this attracting set, so reconstruction is evaluated on the portion of state space that carries mass under the corresponding invariant measure of this set. 
\paragraph{Conventions used throughout this work.}
We follow the state-space notation $\bm{z}_t\in\mathbb{R}^M$ (latent state) and denote observations by $\bm{x}_t\in\mathbb{R}^N$. We choose the measurement operator $g$ to be a linear projection
\begin{equation}
\label{eq:a0:obs_linear_B}
\bm{x}_t = \bm{B}\bm{z}_t + \bm{\eta}_t,\qquad \bm{\eta}_t\sim\mathcal N(\bm 0,\bm R).
\end{equation}
For identity teacher forcing (ITF), we use $\bm{B}=[\bm{I}_N\ \bm{0}] \in \mathbb{R}^{N\times M}$ to select the first $N$ coordinates of the state as observations. At forcing times $\mathcal T_\tau=\{t:t\equiv 0\ (\mathrm{mod}\ \tau),\,t>0\}$, ITF applies the forcing intervention
\begin{equation}
\label{eq:a0:itf_overwrite}
\tilde{\bm z}_t=\bm z_t+\bm B^{\mathsf T}(\bm x_t-\bm B\bm z_t),
\end{equation}
which overwrites the observed coordinates with the data while leaving the unobserved coordinates unchanged.
\subsection{Measuring success: dynamical QoIs}
The success of DSR is typically judged by whether the free-running surrogate reproduces the system's \emph{typical long-run} behavior on the attractor, rather than by trajectory-level alignment. Under standard ergodicity/mixing assumptions, long-run statistics (time averages of observables) can be expressed as expectations under the system's invariant measure. Concretely, an \emph{invariant measure} $\mu$ is a stationary distribution over states that satisfies $\mu(F^{\star -1}(Z))=\mu(Z)$ for all measurable sets $Z \subseteq \mathcal{Z}$. Let $\mu$ denote the invariant measure of the true system on the attractor of interest, and $\mu_{\bm\theta}$ the invariant measure induced by the surrogate. Matching long-run properties can then be expressed as agreement of expectations of relevant observables $\psi$,
\begin{equation}
\label{eq:a0:inv_expect}
\E_{\bm{z}\sim\mu}[\psi(\bm{z})]\ \approx\ \E_{\bm{z}\sim\mu_{\bm\theta}}[\psi(\bm{z})].
\end{equation}
In practice, agreement in Eq.~\eqref{eq:a0:inv_expect} is assessed by comparing long observed sequences to long free-running surrogate rollouts. Examples of $\psi$ include (i) overlap in long-run state occupancy (a proxy for agreement between invariant measures), (ii) temporal summary statistics (e.g., power spectrum), and (iii) Lyapunov exponents (LEs). Typically, one discards an initial transient (``burn-in'') from rollouts to reduce dependence on initialization.\footnote{Early in a rollout, the state distribution can still depend strongly on the initial condition (before the trajectory has mixed onto the attracting set). Discarding an initial ``burn-in'' reduces this initialization bias so that empirical averages better approximate invariant-measure expectations.}

\subsection{Challenges of DSR for chaotic systems}
Chaotic systems exhibit sensitive dependence on initial conditions: first-order perturbations typically grow like
\begin{equation}
\label{eq:a0:lyap_growth}
\delta\bm z_{T+1} \approx \big(\prod_{t=1}^{T} \bm J_t\big)\,\delta\bm z_1, \quad \text{where } \bm J_t := \partial F^{\star}(\bm z_t)/\partial \bm z_t.
\end{equation}
The largest LE of the system (``maximal LE'') is then given by the exponential growth rate of these Jacobian products,
\begin{equation}
	\lambda_{\max} = \limsup_{T\to\infty}\frac{1}{T}\log\|\prod_{t=1}^{T}\bm J_t\|.
\end{equation}
Chaotic systems have a positive maximal LE, $\lambda_{\max}>0$, so small initial errors grow exponentially, making long-term trajectory alignment impossible in practice. This implies a finite effective forecast horizon \emph{even for a correct model}. This has two important consequences. First, trajectory-alignment losses (e.g. MSE over long horizons) become uninformative for DSR on chaotic systems. Second, the same Jacobian products that drive chaotic expansion also appear in backpropagation through time (BPTT) and real time recurrent learning, creating exploding gradients that make naive long-horizon training unstable \citep{mikhaeil2022chaosrnn}. A common rough forecast-horizon scale is the Lyapunov time $t_\lambda:=1/\lambda_{\max}$, i.e. a characteristic time over which small errors amplify by a factor $e$.
\subsection{Teacher forcing as intervention-based training}
To alleviate the ill-conditioning of BPTT in chaotic DSR, modern DSR training methods introduce controlled interventions to stabilize gradients for recurrent neural networks (RNNs) while still training the model to generate plausible long rollouts. Sparse teacher forcing (STF) and generalized teacher forcing (GTF) are two prominent examples \citep{mikhaeil2022chaosrnn,hess2023gtf}.
In GTF, one interpolates between model state and a data-inferred state
\begin{equation}
\label{eq:a0:gtf_mix}
\hat{\bm{z}}_t = (1-\alpha_t)\bm{z}_t + \alpha_t\,\tilde{\bm{z}}_t,\qquad \tilde{\bm{z}}_t\approx \mathcal{E}(\bm{x}_t),
\end{equation}
and advances the dynamics using this interpolated state (e.g., $\bm z_{t+1}=F_{\bm\theta}(\hat{\bm z}_t)$), with $\alpha_t\in[0,1]$ controlling forcing strength (typically, $\alpha_t\equiv\alpha$ and forcing is applied at every step $t$). Here $\mathcal{E}$ denotes an encoder mapping observations to a surrogate latent state.
STF (including ITF) can be viewed as a sparse special case where the intervention is applied only at forcing times, together with a partial overwrite of observed coordinates.
Both schemes can be understood as optimizing a loss \emph{under an intervention policy} that modifies the autonomous dynamics during training. Crucially, the geometry (curvature) of the resulting training objective need not match the geometry of any likelihood for the free-running model.

\subsection{AL-RNN surrogates for DSR}
\subsubsection{Piecewise-linear RNNs for DSR}
DSR models are typically not only prediction tools but also objects of downstream analysis. Piecewise-linear RNNs have been shown to be an attractive choice because they can represent complex nonlinear dynamics while retaining locally linear structure. This structure leads to closed-form Jacobians and dynamical objects (e.g., fixed points, cycles), thus enabling mechanistic interpretation of the learned dynamics \citep{eisenmann2023bifurcationsrnn,eisenmann2026invariantmanifolds}. The interpretability of these models has been exploited to analyze the dynamics of trained surrogates and to relate them to known mechanisms in the target system, e.g.\ in neuroscience applications
\citep{durstewitz2023reconstructing,emonds2025control}.

\subsubsection{AL-RNNs}
Almost-linear RNNs (AL-RNNs) particularly emphasize interpretability by concentrating the nonlinearity into a small block of $P < M$ ReLU-gated coordinates, yielding a parsimonious representation of nonlinear dynamics with $2^P$ linear subregions and an explicit symbolic regime code \citep{brenner2024alrnn,brenner2026uncovering}.
Formally, the AL-RNN models the transition map by 
\begin{equation}
\label{eq:a:alrnn}
\bm{z}_{t+1}
= F_{\bm\theta}(\bm{z}_t)
:= \bm{A} \bm{z}_t + \bm{W}\,\bm\phi^\ast(\bm{z}_t) + \bm{h},
\end{equation}
using a linear part $\bm{A} \in \mathbb{R}^{M\times M}$ (diagonal), a gated part $\bm{W} \in \mathbb{R}^{M\times M}$ (full), and an affine bias $\bm{h} \in \mathbb{R}^M$. The activation $\bm\phi^\ast$ applies ReLU only on the last $P$ coordinates and is identity elsewhere. While external inputs can also be included into the state transition, we focus on autonomous dynamics throughout this work. Let $\bm{c}_t\in\{0,1\}^P$ be the 
\emph{switching code} that records which of the last $P$ coordinates are active (positive) at time $t$. This allows us to rewrite the transition as a switching affine system,
\begin{equation}
\label{eq:alrnnswitching}
\bm z_{t+1} = \big(\bm A+\bm W\bm D(\bm c_t)\big)\bm z_t + \bm h,
\qquad
\bm{D}(\bm{c}_t)
=\mathrm{diag}\big( \underbrace{1,\dots,1}_{M-P},\,c_{t,1},\dots,c_{t,P} \big).
\end{equation}
Hence conditional on $\bm c_t$, the dynamics are linear. The switching code sequence
$\{\bm c_t\}_{t=1:T}$ can be analyzed as a symbolic description of how the model partitions state space (e.g., regime occupancy, switching rate, regime-specific linearization and stability).

\paragraph{Learned linear initialization embedding.}
When initializing an AL-RNN rollout from the first observation $\bm x_1\in\mathbb{R}^N$ (rather than from a full latent initial state $\bm z_1\in\mathbb{R}^M$), we embed $\bm x_1$ into the full state space using a learned linear map $\bm E\in\mathbb{R}^{M\times N}$,
\begin{equation}
\label{eq:alrnn_init_embed}
\bm z_1 \leftarrow \bm E\bm x_1,\qquad \bm z_1 \leftarrow \bm z_1 + \bm B^{\mathsf T}(\bm x_1-\bm B\bm z_1),
\end{equation}
so that the observed coordinates match the data ($\bm B\bm z_1=\bm x_1$) while the unobserved coordinates are provided by the initializer.
The initializer $\bm E$ is trained jointly with the AL-RNN parameters under the same training objective (i.e., it is a linear map that produces an initial latent state from the first observation).

\subsection{Why objective geometry matters in DSR}
For the goals of DSR,
there are multiple reasonable learning criteria (teacher-forced prediction losses, likelihood-based objectives with latent states, and hybrid criteria), and they can disagree strongly in which parameter directions they treat as sensitive. Curvature provides a local summary of this geometry, quantifying ``which directions matter''. In particular, under chaos, small changes in parameters lead to large changes in trajectories, so the curvature of the objective can have a strong influence on which local minima are found and how well they match the true system's invariant structure. In switching latent-variable models, Louis' identity implies that latent ambiguity reduces marginal-likelihood curvature through a missing-information correction term. This is the mechanism behind our ``geometry mismatch'' narrative: training-time forcing can induce sharp curvature (overconfident geometry) when the marginal likelihood, after properly accounting for switching uncertainty, is comparatively flat.

\paragraph{QoI-aware objectives.}
Related approaches in data-driven surrogate modeling emphasize matching long-horizon statistics (e.g., invariant measures and distributional summaries) via losses defined on long rollouts or on learned operators acting on observables \citep{schiff2024dyslim,cheng2025learningcl}. These objectives are complementary to the local-geometry perspective used here: we focus on how training-time forcing and latent switching uncertainty shape the curvature of the fitted criterion, and how this geometry can disagree with ambiguity-aware likelihood curvature even when the resulting rollouts appear qualitatively plausible.

\section{Methodological details}
\label{sec:s:method_details}

\subsection{Louis missing-information mechanism for two-state gates}
\label{sec:s:louis_toy}
This section states the two-state Louis mechanism used for intuition in the main paper, based on Louis' identity \citep{louis1982,dempster1977em}. It also specifies the controlled-ambiguity toy model used to visualize the mechanism.

For completeness, we restate Louis' identity in standard latent-variable notation.
Let $\bm z$ denote latent variables and define the complete-data log density $\ell_{\mathrm{comp}}(\bm\theta):=\log p_{\bm\theta}(\bm x,\bm z)$.
Then the observed information for the marginal likelihood $p_{\bm\theta}(\bm x)$ satisfies
\begin{equation}
\label{eq:s:louis_identity}
\Io(\bm\theta)
= -\nabla_{\bm\theta}^2\log p_{\bm\theta}(\bm x)
= \E\big[-\nabla_{\bm\theta}^2\ell_{\mathrm{comp}}(\bm\theta)\mid \bm x\big]
- \Cov\big(\nabla_{\bm\theta}\ell_{\mathrm{comp}}(\bm\theta)\mid \bm x\big),
\end{equation}
where the conditional moments are taken under the posterior $p_{\bm\theta}(\bm z\mid \bm x)$.

\begin{proposition}[Louis missing-information term under state-dependent gating]
\label{prop:s:louis_missing_state_dependent_gate}
Let $C$ be a latent gate variable and let $\bm{x}$ denote the observed quantity.
For brevity, we write
\begin{equation}
p_{\bm\theta}(\bm x,c):=p_{\bm\theta}(\bm x, C{=}c),\qquad
p_{\bm\theta}(\bm x\mid c):=p_{\bm\theta}(\bm x\mid C{=}c),\qquad
p_{\bm\theta}(c\mid \bm x):=\mathbb{P}_{\bm\theta}(C{=}c\mid \bm x).
\end{equation}
Assume the joint model can be written as
\begin{equation}
p_{\bm\theta}(\bm x,c) \,=\, \pi(c\mid \bm x)\,p_{\bm\theta}(\bm x\mid c),\qquad c\in\{0,1\},
\end{equation}
where $\pi(c\mid \bm x)$ is a (possibly $\bm x$-dependent) gating term that is \emph{independent of $\bm\theta$} for the parameter block of interest.
Define, for each $c\in\{0,1\}$,
\begin{equation}
\bm s_c(\bm\theta) := \nabla_{\bm\theta}\log p_{\bm\theta}(\bm x\mid c),
\qquad
\bm{\mathcal I}_c(\bm\theta) := -\nabla_{\bm\theta}^2\log p_{\bm\theta}(\bm x\mid c).
\end{equation}
Let
\begin{equation}
p := p_{\bm\theta}(1\mid \bm x)
= \frac{\pi(1\mid \bm x)\,p_{\bm\theta}(\bm x\mid 1)}{\pi(0\mid \bm x)\,p_{\bm\theta}(\bm x\mid 0)+\pi(1\mid \bm x)\,p_{\bm\theta}(\bm x\mid 1)}.
\end{equation}
Then the observed information for the marginal likelihood $p_{\bm\theta}(\bm x)$ satisfies
\begin{equation}
\label{eq:s:2switch_state_dependent}
\Io(\bm\theta)
=
p\,\bm{\mathcal{I}}_1(\bm\theta) + (1-p)\,\bm{\mathcal{I}}_0(\bm\theta)
- \Imiss(\bm\theta).
\end{equation}
\begin{equation}
\label{eq:s:2switch_state_dependent_miss}
\Imiss(\bm\theta)
:=
p(1-p)\big(\bm{s}_1(\bm\theta)-\bm{s}_0(\bm\theta)\big)\big(\bm{s}_1(\bm\theta)-\bm{s}_0(\bm\theta)\big)^{\mathsf T}
\succeq 0.
\end{equation}
In particular, the missing-information term $\Imiss(\bm\theta)$ is large when (i) the posterior is ambiguous
($p(1-p)$ is maximized at $p=1/2$) and (ii) the two explanations disagree in score
(large $\|\bm{s}_1(\bm\theta)-\bm{s}_0(\bm\theta)\|$).
\end{proposition}

\begin{remark}[When gating depends on $\bm\theta$]
If the gating distribution depends on $\bm\theta$, then the complete-data score for $(\bm x, C)$ includes an additional term
$\nabla_{\bm\theta}\log \pi_{\bm\theta}(C\mid \bm{x})$, so both terms in Louis’ identity acquire extra contributions.
In that case, the simple rank-one form in Eq.~\eqref{eq:s:2switch_state_dependent} need not hold exactly, but the qualitative mechanism remains:
greater posterior ambiguity and larger between-component score gaps reduce observed information.
\end{remark}

\subsubsection{Proof}
\begin{proof}[Proof of Proposition~\ref{prop:s:louis_missing_state_dependent_gate}]
Because $\pi(c\mid \bm x)$ is $\bm\theta$-independent, we have
\begin{equation}
\log p_{\bm\theta}(\bm x,c)=\log \pi(c\mid \bm x)+\log p_{\bm\theta}(\bm x\mid c).
\end{equation}
Therefore the complete-data score and complete-data negative Hessian for the parameter block $\bm\theta$ are
\begin{equation}
\nabla_{\bm\theta}\log p_{\bm\theta}(\bm x,c)=\bm s_c(\bm\theta),
\qquad
-\nabla_{\bm\theta}^2\log p_{\bm\theta}(\bm x,c)=\bm{\mathcal I}_c(\bm\theta).
\end{equation}
Louis' identity for this two-point latent gate can be written as
\begin{align}
\Io(\bm\theta)
&=\sum_{c\in\{0,1\}}p_{\bm\theta}(c\mid \bm x)\,\bm{\mathcal I}_c(\bm\theta)\notag\\
&\quad-\Bigg(
\sum_{c\in\{0,1\}}p_{\bm\theta}(c\mid \bm x)\,\bm s_c(\bm\theta)\bm s_c(\bm\theta)^{\mathsf T}
-\Big(\sum_{c\in\{0,1\}}p_{\bm\theta}(c\mid \bm x)\,\bm s_c(\bm\theta)\Big)\Big(\sum_{c\in\{0,1\}}p_{\bm\theta}(c\mid \bm x)\,\bm s_c(\bm\theta)\Big)^{\mathsf T}
\Bigg).
\end{align}
Writing $p:=p_{\bm\theta}(1\mid \bm x)$ and simplifying yields
\begin{equation}
\Io(\bm\theta)=p\,\bm{\mathcal I}_1(\bm\theta)+(1-p)\,\bm{\mathcal I}_0(\bm\theta)
-p(1-p)\big(\bm s_1(\bm\theta)-\bm s_0(\bm\theta)\big)\big(\bm s_1(\bm\theta)-\bm s_0(\bm\theta)\big)^{\mathsf T},
\end{equation}
which is Eq.~\eqref{eq:s:2switch_state_dependent} with the missing-information term shown in Eq.~\eqref{eq:s:2switch_state_dependent_miss}.
\end{proof}

\subsubsection{Toy model settings (controlled switching ambiguity)}
The mechanism illustrated in the main text (Fig.~\ref{fig:results}a) was computed using a two-regime, scalar switching AR(1) toy model with a probit gate. Let $x_t\in\mathbb{R}$ denote the (observed) scalar state and let $c_t\in\{0,1\}$ denote the latent regime indicator at time $t$.
We generate trajectories from
\begin{equation}
c_t\mid x_t \sim \mathrm{Bernoulli}\!\left(\Phi\!\left(x_t/\sigma_g\right)\right),\qquad
x_{t+1}\mid x_t,c_t \sim \mathcal{N}\!\left(a_{c_t}x_t,\,\sigma^2\right),
\end{equation}
where $\Phi$ is the standard normal CDF and $\sigma_g$ controls gate stochasticity (and therefore posterior switching ambiguity).
We fix $a_0=0.90$, $a_1=0.60$, $\sigma=0.15$, and simulate length-$T$ trajectories with $T=600$ for $20$ random seeds.
All time indices in this toy model use $t=1,\dots,T-1$.

To vary ambiguity, we sweep $\sigma_g$ over $25$ logarithmically spaced values between $0.03$ and $0.9$.
For each simulated pair $(x_t,x_{t+1})$, the gate posterior is available in closed form:
letting $\pi_t:=\Phi(x_t/\sigma_g)$, we have
\begin{equation}
p_t:=\mathbb{P}(c_t=1\mid x_t,x_{t+1})
=\frac{\pi_t\,\mathcal{N}(x_{t+1};a_1x_t,\sigma^2)}{(1-\pi_t)\,\mathcal{N}(x_{t+1};a_0x_t,\sigma^2)+\pi_t\,\mathcal{N}(x_{t+1};a_1x_t,\sigma^2)}.
\end{equation}
The switching-uncertainty summary reported in the main paper (Fig.~\ref{fig:results}a) is the time-average of the Bernoulli entropy (in bits),
\begin{equation}
H(p_t) = \frac{1}{T-1}\sum_{t=1}^{T-1}\big[-p_t\log_2 p_t-(1-p_t)\log_2(1-p_t)\big].
\end{equation}
For the same simulated trajectories we compute the toy-model observed information $\Io$ for the drift-parameter block $\bm\theta=(a_0,a_1)$ by applying the two-state Louis decomposition in Eq.~\eqref{eq:s:2switch_state_dependent} at each time step using the closed-form posteriors $p_t$, and then summing over $t$.
\subsection{Observed information via RBPF and Louis' identity}
\label{sec:s:rbpf_louis_details}
This section describes how we estimate observed information for the \emph{probabilistic AL-RNN (PAL-RNN)}, which augments the AL-RNN to a switching linear-Gaussian state-space model (SSM) with latent gate variables \citep{durbin2001ssm}. The key point is that, conditional on a gate path $\bm c_{1:T-1}$, the continuous state admits closed-form Kalman filtering/smoothing and closed-form complete-data score and information contributions. Thus, we approximate the required smoothing distribution with a Rao--Blackwellized particle filter/smoother (RBPF) over gate paths and then apply Louis' identity to estimate $\Io$. See e.g. \citet{doucet2000rbpf} for RBPF background.
\subsubsection{PAL-RNN}
The PAL-RNN is a probabilistic augmentation of the AL-RNN (Eq.~\eqref{eq:alrnnswitching}) that retains the same switching structure. However, it treats the switching code as a latent random variable via a probit gate model and adds Gaussian process/observation noise, so that conditional on a gate path it becomes a linear-Gaussian SSM.

Figure~\ref{fig:s:palrnn_graphical_model} shows the corresponding directed graphical model.

\begin{figure}[t]
\centering
\includegraphics[width=0.5\textwidth]{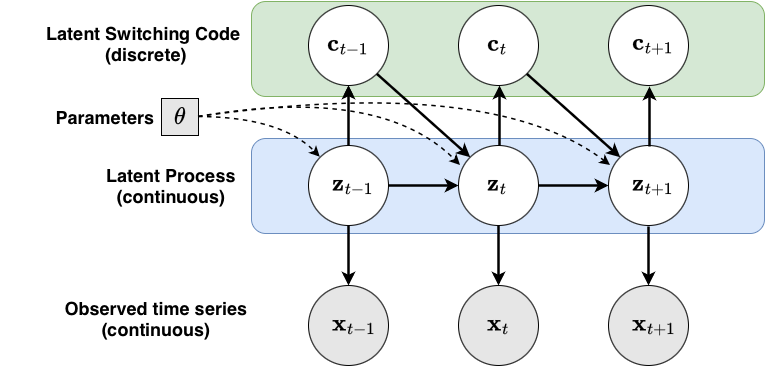}
\caption{Directed graphical model of the probabilistic AL-RNN (PAL-RNN), shown as a conditional Bayesian network over $(\bm z_{1:T},\bm c_{1:T},\bm x_{1:T})$ given shared drift parameters $\bm\theta=(\bm A,\bm W,\bm h)$. The continuous latent state $\bm z_t$ emits observations $\bm x_t$, and the discrete switching code $\bm c_t\in\{0,1\}^P$ is drawn from a state-dependent probit model $p(\bm c_t\mid \bm z_t)$ and indexes the transition $p(\bm z_{t+1}\mid \bm z_t,\bm c_t;\bm\theta)$. Dashed arrows indicate that the same parameters $\bm\theta$ are shared across time. Noise (hyper-)parameters (e.g., $\bm Q,\bm R,\sigma_g$) are treated as fixed and omitted for visual clarity.}
\label{fig:s:palrnn_graphical_model}
\end{figure}

We use the same notation as in Eq.~\eqref{eq:alrnnswitching}. In particular, $\bm c_t\in\{0,1\}^P$ encodes which of the last $P$ coordinates are gated (active) and $\bm D(\bm c_t)$ is the corresponding diagonal gate matrix.
For a length-$T$ segment, gates index transitions for $t=1,\dots,T-1$. The conditional dynamics and observation model are
\begin{align}
\bm{z}_{t+1} &= \big(\bm{A}+\bm{W}\bm{D}(\bm{c}_t)\big)\bm{z}_t + \bm{h} + \bm\varepsilon_t,\qquad \bm\varepsilon_t\sim\mathcal{N}(\bm{0},\,\bm{Q}),\notag
\\
\bm{x}_t &= \bm{B}\bm{z}_t + \bm\eta_t,\qquad \bm\eta_t\sim\mathcal{N}(\bm{0},\,\bm{R}).
\end{align}
We use the linear observation model in Eq.~\eqref{eq:a0:obs_linear_B} with the identity projection $\bm B=[\bm I_N\ \bm 0]$. In our experiments, we use isotropic noise $\bm{Q}=\sigma_{\mathrm{proc}}^2\bm{I}_M$ and $\bm{R}=\sigma_{\mathrm{obs}}^2\bm{I}_N$.

The switching law is state dependent but has no \emph{explicit} dependence on the drift-parameter block of interest $\bm\theta=(\bm A,\bm W,\bm h)$:
\begin{equation}
\mathbb{P}(c_{t,j}=1\mid \bm{z}_t) = \Phi\!\left(\frac{z_{t,M-P+j}}{\sigma_g}\right),\qquad j=1,\dots,P,
\end{equation}
where $\Phi$ is the standard normal CDF and $\sigma_g$ is a fixed gate-noise scale.
For the curvature diagnostics in this section, we hold $(\sigma_{\mathrm{proc}},\sigma_{\mathrm{obs}},\sigma_g)$ fixed as hyperparameters (Table~\ref{tab:s:rbpf_hparams}). In the particle-SAEM experiments later in the supplement (Section~\ref{sec:s:saem_details}), we update $(\sigma_{\mathrm{proc}},\sigma_{\mathrm{obs}})$ in the \texttt{calib} and \texttt{full SAEM} configurations, while keeping $\sigma_g$ fixed.
We also treat the learned linear initialization embedding $\bm E$ (Eq.~\eqref{eq:alrnn_init_embed}) as fixed when computing curvature.
Under these assumptions, the switching enters Louis’ identity only through posterior uncertainty over gate paths, matching the setting of Proposition~\ref{prop:s:louis_missing_state_dependent_gate}.
\paragraph{Relation to (r)SLDS.}
The PAL-RNN is closely related to switching linear dynamical systems (SLDS) and recurrent SLDS variants used for hybrid time-series modeling \citep{ghahramani2000switchingssm,murphy1998switchingkf,linderman2017rslds}. A key difference is that we do not introduce a separate Markov transition model over discrete regimes: instead, the latent gate bits $\bm c_t\in\{0,1\}^P$ are state dependent and (given $\bm z_t$) conditionally independent across time, reflecting the AL-RNN's per-unit gating structure. Moreover, SLDS/rSLDS typically use a categorical regime variable with regime-specific linear dynamics, whereas PAL-RNN uses a factorial code that deterministically masks the same shared drift parameters via $\bm D(\bm c_t)$.
\subsubsection{RBPF filtering, approximate smoothing, and Louis identity}
The RBPF samples gate paths and Rao--Blackwellizes the continuous state using a Kalman filter, yielding a (particle) mixture of Kalman filters \citep{chenliu2000mixturekf}. At each time $t$, particle $i$ carries a sampled gate prefix $\bm c_{1:t-1}^{(i)}$ and a Gaussian filtering belief
\begin{equation}
p(\bm{z}_t\mid \bm{x}_{1:t},\bm{c}_{1:t-1}^{(i)})=\mathcal{N}(\bm{m}_t^{(i)},\bm{P}_t^{(i)}).
\end{equation}
The RBPF represents the gate-prefix posterior $p(\bm c_{1:t-1}\mid \bm x_{1:t})$ with finitely many particles (Monte Carlo approximation). Conditional on a fixed gate prefix, the Kalman update is exact.
Gate bits are proposed with a factorized \emph{integrated probit} rule, i.e. we analytically average the probit link under the particle’s Gaussian belief. Concretely, if a gated coordinate has marginal $z\sim\mathcal{N}(\mu,\nu)$ under the particle’s current belief, then
\begin{equation}
\mathbb{P}(c{=}1\mid \mu,\nu)=\Phi\!\left(\frac{\mu}{\sqrt{\nu+\sigma_g^2}}\right).
\end{equation}
Given sampled $\bm{c}_t^{(i)}$, the conditional dynamics are linear-Gaussian and the particle performs a Kalman predict/update to obtain $(\bm m_{t+1}^{(i)},\bm P_{t+1}^{(i)})$, with importance weights updated by the one-step predictive likelihood.
Resampling is triggered when the effective sample size (ESS) falls below a fraction of $N_p$. With normalized weights $w_t^{(i)}\propto\exp(\log w_t^{(i)})$, the ESS is
\begin{equation}
	\label{eq:rbpf_ess}
	\mathrm{ESS}_t := \frac{1}{\sum_{i=1}^{N_p} (w_t^{(i)})^2},
\end{equation}
and when $\mathrm{ESS}_t < \tau_{\mathrm{ESS}} N_p$, where $\tau_{\mathrm{ESS}}\in(0,1)$ is the resampling threshold (here $\tau_{\mathrm{ESS}}{=}0.5$), we resample particle indices. This is done via i.i.d. sampling from the categorical distribution with probabilities $w_t$ using multinomial resampling with replacement. Afterwards, we reset weights to $1/N_p$, and continue. 

The RBPF used for the experiments employs a bootstrap-style proposal: gate bits are sampled from the integrated gate prior described above, and the weight update uses the Kalman log-likelihood increment $\log p(\bm{x}_{t+1}\mid \bm{x}_{1:t},\bm{c}_{1:t}^{(i)};\bm\theta)$.

To approximate Louis' identity on a length-$T$ segment, we generate $S$ approximate smoothed $(\bm c,\bm z)$ trajectories from the RBPF by tracing ancestors and sampling backward in time under the conditional Gaussian dynamics. For each sampled trajectory, we compute per-step complete-data score contributions $\bm{s}_t(\bm\theta)$ and complete-data information contributions $\bm{\mathcal{I}}_{\mathrm{comp},t}(\bm\theta)$ for the drift-parameter block $\bm\theta=(\bm A,\bm W,\bm h)$.
The observed information estimate then is given by
\begin{align}
\Iohat
&\;=\; \widehat{\E}\!\left[\sum_{t=1}^{T-1}\bm{\mathcal{I}}_{\mathrm{comp},t}\right] \; - \; \widehat{\Cov}\!\left(\sum_{t=1}^{T-1}\bm{s}_t\right)\notag\\
&\;=\; \widehat{\E}[\Ic] \; - \; \widehat{\Cov}\!\left(\sum_{t=1}^{T-1}\bm{s}_t\right).
\end{align}
Here $\widehat{\E}$ and $\widehat{\Cov}$ denote the empirical mean and covariance over the $S$ sampled smoothed trajectories. Hence,
this is Louis’ identity with Monte Carlo estimates of the complete-data expectation and the conditional score covariance under the RBPF smoothing distribution. Pseudo-code for the estimator is given in Algorithm~\ref{alg:rbpf-louis}.

\begin{algorithm}[t]
\caption{RBPF and Louis estimator of observed information}
\label{alg:rbpf-louis}
\begin{algorithmic}[1]
\STATE \textbf{Input:} observations $\bm{x}_{1:T}$; parameters $\bm\theta=(\bm{A},\bm{W},\bm{h})$; hyperparameters $(\sigma_{\mathrm{proc}},\sigma_{\mathrm{obs}},\sigma_g)$; number of particles $N_p$; number of smoothing trajectories $S$.
\STATE \textbf{Init:} for each particle $i$, set $w_1^{(i)}\leftarrow 1/N_p$ and initialize the filter $(\bm{m}_1^{(i)},\bm{P}_1^{(i)})$.
\FOR{$t=1,\dots,T-1$}
	\FOR{$i=1,\dots,N_p$}
		\STATE Sample gate bits $\bm{c}_t^{(i)}$ from the factorized integrated-probit proposal under $\mathcal{N}(\bm{m}_t^{(i)},\bm{P}_t^{(i)})$.
		\STATE Run Kalman predict/update to obtain $(\bm{m}_{t+1}^{(i)},\bm{P}_{t+1}^{(i)})$ given $\bm{c}_t^{(i)}$.
		\STATE Update weight: $w_{t+1}^{(i)}\propto w_t^{(i)}\,p(\bm{x}_{t+1}\mid \bm{x}_{1:t},\bm{c}_{1:t}^{(i)};\bm\theta)$.
	\ENDFOR
		\STATE Normalize weights.
		\STATE If ESS $< \tau_{\mathrm{ESS}} N_p$ (default $\tau_{\mathrm{ESS}}=0.5$), resample and reset weights; store ancestry and sampled codes.
\ENDFOR
\FOR{$s=1,\dots,S$}
	\STATE Sample a gate path $\bm{c}_{1:T-1}^{(s)}$ via ancestry tracing / backward sampling.
	\STATE Sample a continuous trajectory $\bm{z}_{1:T}^{(s)}$ from the conditional Gaussian smoother (given $\bm{x}_{1:T}$ and $\bm{c}_{1:T-1}^{(s)}$).
	\STATE Accumulate complete-data score $\sum_t \bm{s}_t^{(s)}$ and complete-data information $\sum_t \bm{\mathcal{I}}_{\mathrm{comp},t}^{(s)}$ for $\bm\theta=(\bm{A},\bm{W},\bm{h})$.
\ENDFOR
\STATE Return $\Iohat=\widehat{\E}[\Ic] - \widehat{\Cov}(\sum_t \bm{s}_t)$ and diagnostics (ESS, entropies).
\end{algorithmic}
\end{algorithm}

\subsubsection{Filtering switching code entropy and ESS}
We summarize switching uncertainty using the time-averaged RBPF filtering (online) switching code entropy $H_c$ (in bits), computed from the RBPF filtering distribution at each transition time $t=1,\dots,T-1$. We use filtering rather than smoothing entropy because it is the most directly aligned with ITF, which is a forward (online) intervention. Let $w_t^{(i)}$ denote the normalized particle weights after the weight update at transition $t$ (i.e., after incorporating $\bm x_{t+1}$ via the one-step predictive likelihood) and let $\bm{c}_t^{(i)}\in\{0,1\}^P$ be the sampled gate bits.
Equivalently, the induced full-code distribution $\hat p_t(c)$ defined below estimates the online gate uncertainty at transition $t$:
\begin{equation}
\hat p_t(c)\ \approx\ p_{\bm\theta}(\bm c_t=c\mid \bm x_{1:t+1}).
\end{equation}

\paragraph{Full-code entropy.}
Since the size of the full codebook is manageable in our experiments, we can form an empirical filtering posterior over full codes $c\in\mathcal C$, where $\mathcal C:=\{0,1\}^P$, via
\begin{align}
\hat p_t(c) &:= \sum_{i=1}^{N_p} w_t^{(i)}\,\mathbbm{1}\{\bm{c}_t^{(i)}=c\},\notag\\
H_{t,\mathrm{code}} &:= -\sum_{c\in\mathcal C} \hat p_t(c)\log \hat p_t(c).
\end{align}
When $2^P$ is too large to tabulate the full code distribution, one can instead use a bit-marginal entropy sum as a cheaper proxy.

\paragraph{Time-averaged entropy in bits.}
The entropy values reported in the curvature-gap results (Fig.~\ref{fig:results}b) are the time average of the full-code entropy, converted to bits:
\begin{equation}
H_c := \frac{1}{(T-1)\log 2}\sum_{t=1}^{T-1} H_{t,\mathrm{code}}.
\end{equation}

\paragraph{Missing-information ratio (MIR).}
To summarize the magnitude of the Louis missing-information correction relative to the complete-data curvature, we report the missing-information ratio
\begin{align}
\mathrm{MIR} &:= 1-\frac{\tr(\Iohat)}{\tr(\widehat{\E}[\Ic\mid \bm x_{1:T}])}\notag\\
&=\frac{\tr(\widehat{\E}[\Ic\mid \bm x_{1:T}]-\Iohat)}{\tr(\widehat{\E}[\Ic\mid \bm x_{1:T}])}.
\end{align}
where $\widehat{\E}[\Ic\mid \bm x_{1:T}]$ is the Monte Carlo estimate of the complete-data information term from the same smoothing samples used for $\Iohat$. Values closer to $1$ indicate larger curvature reductions due to latent uncertainty.
\subsubsection{Settings used for curvature-gap experiments}
In the curvature-gap experiment (Fig.~\ref{fig:results}b), the RBPF/Louis diagnostics are computed on fixed-length segments. The hyperparameter settings are summarized in Table~\ref{tab:s:rbpf_hparams}.
\begin{table}[!htbp]
\centering
\caption{Settings used for the curvature-gap analysis in Fig.~\ref{fig:results}b.}
\label{tab:s:rbpf_hparams}
{\small\setlength{\tabcolsep}{4pt}\renewcommand{\arraystretch}{1.1}
\begin{tabular}{l l}
\hline
Setting & Value \\
\hline
\multicolumn{2}{l}{\textbf{AL-RNN training settings}}\\
Dataset & Lorenz-63 \\
Training trajectory length $T_{\mathrm{train}}$& $8 \cdot 10^4$ \\
Noise regimes & $\sigma_{\mathrm{proc}}\in\{0.1,0.3,0.5\}$, $\sigma_{\mathrm{obs}}\in\{0.1, 0.3, 0.5\}$ \\
Forcing interval $\tau$ & $\{4,8,16,32,64\}$ \\
Latent dimension $M$ & 30 \\
Number of gated units $P$ & 10 \\
Batch size $B$ & 16 \\
BPTT sequence length $L_{\mathrm{BPTT}}$ & 200 \\
Epochs & 2000\\
Batches per epoch & 50 \\
\hline
\multicolumn{2}{l}{\textbf{RBPF/Louis settings}}\\
Sequence length for curvature diagnostics $T$ & 200 \\
Particles $N_p$ & 64 \\
Smoothing samples $S$ & 8 \\
Resampling threshold $\tau_{\mathrm{ESS}}$ & 0.5 \\
Gate noise $\sigma_g$ & 0.1\\
\hline
\end{tabular}}
\end{table}
\subsection{ITF-aligned curvature proxy}
\label{sec:s:tf_fim}
This section gives the closed-form computation of the ITF-aligned generalized Gauss--Newton/Fisher curvature proxy used throughout the experiments.

Under identity teacher forcing (ITF), forcing is applied at times $\mathcal T_\tau=\{t:t\equiv 0\ (\mathrm{mod}\ \tau),\,t>0\}$, as described in Appendix~\ref{sec:a:primer} (Eq.~\eqref{eq:a0:itf_overwrite}). We compute the per-step generalized Gauss--Newton/Fisher matrix aligned with this intervention-based training objective, denoted $\bm{\mathcal{I}}_{\mathrm{ITF}}$, by differentiating the forced rollout and accumulating the resulting Jacobians. We use $\tr(\bm{\mathcal{I}}_{\mathrm{ITF}})$ as a curvature proxy for ITF training, and as a comparator to ambiguity-aware observed information (Section~\ref{sec:s:rbpf_louis_details}).

\paragraph{Gaussian pseudo-likelihood aligned with ITF.}
Let $\{\bm{x}_t\}_{t=1}^{T}$ be an observed trajectory with $\bm{x}_t\in\mathbb{R}^N$ and let $\hat{\bm x}_t:=\bm B\bm z_t$ be the one-step prediction of the observed coordinates. The ITF loss used for training is the MSE along the forced rollout,
\begin{equation}
\label{eq:a:tf_loss}
\mathcal{L}_{\mathrm{ITF}}(\bm\theta;\bm{x}_{1:T})
= \frac{1}{T-1}\sum_{t=1}^{T-1}\big\|\bm{B}\bm{z}_{t+1}-\bm{x}_{t+1}\big\|_2^2.
\end{equation}
For curvature diagnostics, we interpret this as a Gaussian negative log pseudo-likelihood with an isotropic weighting $\bm\Lambda=\sigma_{\mathrm{obs}}^2\bm I_N$, so that the generalized Gauss--Newton matrix equals the Fisher matrix for this pseudo-likelihood.

\paragraph{ITF intervention and sensitivity reset.}
Let $\bm{z}_t\in\mathbb{R}^M$ be the latent state and let $\bm B=[\bm I_N\ \bm 0]$ be the projection from Eq.~\eqref{eq:a0:obs_linear_B}. At forcing times, ITF overwrites the observed coordinates,
\begin{align}
\tilde{\bm z}_t
&=\bm z_t+\bm B^{\mathsf T}(\bm x_t-\bm B\bm z_t)\notag\\
&=\bm M\bm z_t+\bm B^{\mathsf T}\bm x_t,\qquad
\bm M:=\bm I-\bm B^{\mathsf T}\bm B.
\end{align}
Since $\bm x_t$ does not depend on parameters, this implies the sensitivity boundary condition
\begin{equation}
\frac{\partial\tilde{\bm z}_t}{\partial\bm\theta}=\bm M\frac{\partial\bm z_t}{\partial\bm\theta}.
\end{equation}
Thus, at forcing times, the sensitivity $\partial \bm{z}_t/\partial \bm\theta$ is reset by multiplication with $\bm M$, which zeros out the first $N$ rows (i.e. the observed coordinates) and leaves the remaining $M-N$ rows unchanged.

\paragraph{Piecewise-linear derivatives along the forced rollout.}
The forced rollout evolves according to the AL-RNN transition map evaluated at the forced or unforced state,
\begin{equation}
\bm{z}_{t+1} = F_{\bm \theta}(\bar{\bm z}_t)
= \bm{A}\bar{\bm z}_t + \bm{W}\,\bm\phi^\ast(\bar{\bm z}_t) + \bm{h},
\qquad \bar{\bm z}_t:=\begin{cases}\tilde{\bm z}_t,& t\in\mathcal T_\tau,\\ \bm z_t,&\text{else.}\end{cases}
\end{equation}
Let $\bm c_t\in\{0,1\}^P$ denote the induced switching code at time $t$ along this rollout, i.e.\ \;$c_{t,j}=\mathbbm{1}\{\bar z_{t,M-P+j}>0\}$. With $\bm D(\bm c_t)$ as in Eq.~\eqref{eq:alrnnswitching}, the Jacobian of the one-step map with respect to the state is
\begin{equation}
\label{eq:a1:stepjacobian}
\bm{J}_t := \frac{\partial F_{\bm \theta}}{\partial \bm{z}}(\bar{\bm z}_t)
= \bm{A} + \bm{W}\bm D(\bm c_t),
\end{equation}
where we adopted the subgradient convention of setting derivatives at the ReLU-kink to zero.
\paragraph{Closed-form sensitivity recursion for $\partial \bm{z}_t/\partial \bm\theta$.}
We evaluate curvature for the drift-parameter block $\bm\theta=(\bm A,\bm W,\bm h)$ and treat the initial condition as fixed. (In our implementation, $\bm z_1$ is obtained from a learned linear embedding of the first observation, which we keep fixed for curvature diagnostics, see Eq.~\eqref{eq:alrnn_init_embed}).

For the sensitivity recursion, we parameterize $\bm A$ by its diagonal entries $\bm a:=\mathrm{diag}(\bm A)\in\mathbb R^M$ and vectorize
\begin{equation}
\bm\theta := \begin{bmatrix}\bm a\\ \mathrm{vec}(\bm{W})\\ \bm{h}\end{bmatrix}\in\mathbb{R}^{p},
\qquad p=M+M^2+M,
\end{equation}
using row-major vectorization for $\mathrm{vec}(\bm{W})$. Define the sensitivity matrix $\bm{S}_t := \partial \bm{z}_t/\partial \bm\theta\in\mathbb{R}^{M\times p}$ and initialize $\bm{S}_1=\bm{0}$. For each step $t=1,\dots,T-1$ we compute
\begin{equation}
\bm{S}_{t+1} = \bm{J}_t\,\bar{\bm S}_t + \bm{V}_t,
\qquad
\bar{\bm S}_t:=\begin{cases}\bm M\bm S_t,& t\in\mathcal T_\tau,\\ \bm S_t,&\text{else,}\end{cases}
\end{equation}
where $\bm{V}_t:=\partial F_{\bm \theta}/\partial \bm\theta$ has the following closed form at state $\bar{\bm z}_t$:
\begin{align}
\text{(a block)}\quad & \bm{V}_t^{(a)} = \mathrm{diag}(\bar{\bm z}_t)\in\mathbb{R}^{M\times M},
\notag\\
\text{(W block)}\quad & \bm{V}_t^{(W)} = \bm{I}_M\otimes \bm\phi^\ast(\bar{\bm z}_t)^{\mathsf T}\in\mathbb{R}^{M\times M^2},
\notag\\
\text{(h block)}\quad & \bm{V}_t^{(h)} = \bm{I}_M\in\mathbb{R}^{M\times M}.
\end{align}
Thus $\bm{V}_t=[\bm{V}_t^{(a)}\;\bm{V}_t^{(W)}\;\bm{V}_t^{(h)}]$.

\paragraph{Information accumulation and per-step normalization.}
The Jacobian of the one-step prediction $\hat{\bm{x}}_{t+1}=\bm B\bm z_{t+1}$ with respect to parameters is $\bm{B}\bm{S}_{t+1}\in\mathbb{R}^{N\times p}$. For the isotropic pseudo-likelihood weighting $\bm\Lambda=\sigma^2\bm I_N$, the ITF-aligned Fisher is
\begin{align}
\bm{\mathcal{I}}_{\mathrm{ITF}} 
&= \frac{1}{T-1}\sum_{t=1}^{T-1} \left(\bm{B}\bm{S}_{t+1}\right)^{\mathsf T}\bm\Lambda^{-1}\left(\bm{B}\bm{S}_{t+1}\right)\notag\\
&= \frac{1}{(T-1)\,\sigma^2}\sum_{t=1}^{T-1} \left(\bm{B}\bm{S}_{t+1}\right)^{\mathsf T}\left(\bm{B}\bm{S}_{t+1}\right),
\end{align}
which is positive semidefinite by construction (up to floating-point roundoff). For notational simplicity, we write $\bm{\mathcal{I}}_{\mathrm{ITF}}$ for this per-step-normalized ITF-aligned Fisher in the main text and the remainder of this supplement.

\section{Datasets and experimental settings}
\label{sec:s:datasets}
\subsection{Lorenz-63 dataset}
All AL-RNN checkpoints analyzed in this work (both for curvature-gap diagnostics and as initializations in the SAEM experiments) are trained on trajectories generated from the Lorenz-63 system \citep{lorenz1963deterministic} across a sweep of observation- and process-noise regimes.
The continuous-time Lorenz-63 dynamics are
\begin{align}
\frac{\mathrm d z_1}{\mathrm d t} &= \sigma\,(z_2-z_1),
&
\frac{\mathrm d z_2}{\mathrm d t} &= z_1(\rho-z_3)-z_2,
&
\frac{\mathrm d z_3}{\mathrm d t} &= z_1 z_2-\beta z_3,
\end{align}
with the standard chaotic parameter setting $\sigma=10$, $\rho=28$, and $\beta=8/3$.
We denote the corresponding discrete-time latent state by $\bm z_t\in\mathbb R^3$ (with components corresponding to $(z_1,z_2,z_3)$).
We discretize with step size $\Delta t=0.01$.
In the deterministic regime ($\sigma_{\mathrm{proc}}=0$), we integrate the ODE using the fourth order Runge-Kutta method (RK4) \citep{kutta1901rungekutta}.
In the stochastic regime ($\sigma_{\mathrm{proc}}>0$), we use an additive-noise SDE discretization implemented as an RK4 drift step plus Euler--Maruyama diffusion,
\begin{equation}
\bm z_{t+1}=\mathrm{RK4}_{\Delta t}(\bm z_t)
+\sigma_{\mathrm{proc}}\sqrt{\Delta t}\,\mathrm{diag}(\bm s_{\mathrm{ref}})\bm\varepsilon_t,
\qquad \bm\varepsilon_t\sim\mathcal N(\bm 0,\bm I_3),
\end{equation}
where $\bm z_t$ denotes the latent Lorenz state and $\bm s_{\mathrm{ref}}$ is the per-coordinate standard deviation of the deterministic reference trajectory. Thus $\sigma_{\mathrm{proc}}$ is reported in standardized-coordinate units, although the perturbation is applied before standardization. 

Observation noise is added in raw coordinates as
\begin{equation}
\bm x_t=\bm z_t+\sigma_{\mathrm{obs}}\,\mathrm{diag}(\bm s_{\mathrm{train}})\bm\eta_t,
\qquad \bm\eta_t\sim\mathcal N(\bm 0,\bm I_3),
\end{equation}
where $\bm s_{\mathrm{train}}$ is the per-coordinate standard deviation used by the training standardizer. After standardization, $\sigma_{\mathrm{obs}}$ is therefore also on a standardized-coordinate noise scale.

\subsection{AL-RNN training settings}
All ITF-trained model checkpoints analyzed in this work (both for curvature-gap diagnostics and as initializations in the SAEM experiments) are trained on the Lorenz-63 time series generated by the procedure above. 

For training and evaluation, we standardize observations per dimension to zero mean and unit variance using the mean and standard deviation computed on the training trajectory. We denote the interleaving forcing interval by $\tau$. The loss minimized is MSE on the next-step prediction of the observed coordinates along the interleaved rollout, as described in Appendix~\ref{sec:a:primer} (Eq.~\eqref{eq:a:tf_loss}).

\paragraph{Parameter initialization.}
For each random initialization seed, we initialize the drift parameters by setting $\bm h=\bm 0$, drawing entries of $\bm W$ i.i.d.\ from $\mathcal N(0,0.01)$, and taking the diagonal entries of $\bm A$ from the diagonal of a random symmetric positive definite matrix normalized to have largest-magnitude eigenvalue 1.
The learned initialization embedding matrix $\bm E\in\mathbb R^{M\times N}$ is initialized by drawing i.i.d. from $U(-1/\sqrt{N},\,1/\sqrt{N})$. 

\paragraph{BPTT training details.}
Training is performed by BPTT on mini-batches of contiguous subsequences sampled from the standardized training trajectory $\bm x_{1:T_{\mathrm{train}}}$. Concretely, each training update uses a batch of $B$ sequences (here $B=16$) of length $L_{\mathrm{BPTT}}$ (here $L_{\mathrm{BPTT}}=200$) drawn uniformly at random (with replacement) from the training trajectory, and we run a fixed number of batches per epoch (here 50).
We optimize the ITF loss with the RAdam optimizer \citep{liu2020radam} using default momentum parameters $\beta_1=0.9$, $\beta_2=0.999$ and $\epsilon=10^{-8}$, and an exponential learning rate decay schedule from $10^{-3}$ to $10^{-5}$ over the course of training. 

The training settings and the subset of regimes used in the curvature-gap figure are summarized in Table~\ref{tab:s:rbpf_hparams}.

\subsection{Particle-SAEM fine-tuning for windowed evidence}
\label{sec:s:saem_details}
This section documents the particle-SAEM procedure used for the analysis of windowed evidence and dynamical QoIs in Fig.~\ref{fig:results}c. The results in Fig.~\ref{fig:results}c were aggregated over 12 ITF-pretrained AL-RNN initializations, one per training noise, each evaluated at five window lengths. The settings used in the experiments are summarized in Table~\ref{tab:s:saem_hparams}.

\subsubsection{Dataset used in SAEM experiments}
For the SAEM fine-tuning experiments we simulate deterministic Lorenz-63 data ($\sigma_{\mathrm{proc}}=0$) for $T=5000$ steps, add i.i.d.\ observation noise with standard deviation $\sigma_{\mathrm{obs}}=0.085$, and stack $n_{\mathrm{seq}}=4$ independent sequences.
We remove potential transients by discarding the first 1000 time steps and standardize each observed coordinate to zero mean and unit variance using statistics pooled over all sequences and time indices.

\subsubsection{Windowed evidence objective}
Given a sequence $\bm x_{1:T}$ (i.e. one of the $n_{\mathrm{seq}}$ Lorenz-63 sequences), we partition it into non-overlapping windows of length $L\in\mathbb N$ with stride $L$.
Concretely, for window starts
\begin{equation}
t_k := 1+(k-1)L,\qquad k=1,\dots,n_{\mathrm{win}},\qquad \text{with }\; t_k+L\le T,
\end{equation}
we define the $k$th window as the observation block $\bm x_{t_k+1:t_k+L}$ together with its left boundary condition $\bm x_{t_k}$.

For a PAL-RNN parameter setting $\bm\theta$ (drift parameters plus any noise parameters being optimized in the chosen configuration), the windowed evidence objective is the conditional log-evidence
\begin{equation}
\log p_{\bm\theta}\!\big(\bm x_{t+1:t+L}\mid \bm x_t\big),
\end{equation}
where the conditioning on $\bm x_t$ fixes the window boundary condition through the learned initialization described in Eq.~\eqref{eq:alrnn_init_embed}. In practice we approximate this quantity with the RBPF estimate of the particle filter normalizing constant,
\begin{equation}
\log p_{\bm\theta}\!\big(\bm x_{t+1:t+L}\mid \bm x_t\big)
\approx \log \widehat Z_{\bm\theta}(t,L),
\end{equation}
where $\widehat Z_{\bm\theta}(t,L)$ is obtained from the same RBPF forward pass used for smoothing (Section~\ref{sec:s:rbpf_louis_details}). Let $w_{t+s-1}^{(i)}$ denote the normalized particle weights before incorporating $\bm x_{t+s}$, and let
\begin{equation}
\ell_{t+s}^{(i)}
:= p_{\bm\theta}\!\big(\bm x_{t+s}\mid \bm x_t,\bm x_{t+1:t+s-1},\bm c_{t:t+s-1}^{(i)}\big),
\end{equation}
denote the one-step predictive likelihood under particle $i$ (available in closed form because, conditional on $\bm c_{t:t+s-1}^{(i)}$, the model is linear-Gaussian). These are the same increments used in the RBPF weight update. The corresponding particle-filter normalizing-constant estimator over the window is
\begin{align}
\widehat Z_{\bm\theta}(t,L)
&:= \prod_{s=1}^{L} \widehat p_{\bm\theta}\!\big(\bm x_{t+s}\mid \bm x_t,\bm x_{t+1:t+s-1}\big),\notag\\
\widehat p_{\bm\theta}\!\big(\bm x_{t+s}\mid \bm x_t,\bm x_{t+1:t+s-1}\big)
&:= \sum_{i=1}^{N_p} w_{t+s-1}^{(i)}\,\ell_{t+s}^{(i)}.
\end{align}
This sequential Monte Carlo estimator targets the desired conditional marginal likelihood and becomes accurate as particle budgets increase \citep{doucet2000smc}. We therefore use $\log\widehat Z_{\bm\theta}(t,L)$ as an  approximation to the windowed conditional log-evidence.

In Fig.~\ref{fig:results}c we report $\log\widehat{Z}_{\bm\theta}(t,L)$ on held-out windows, averaged over windows and normalized per step and per observed dimension,
\begin{equation}
\label{eq:a2:evidence}
	\mathrm{Evidence}(\bm\theta)
	:= \frac{1}{n_{\mathrm{win}}\,L\,N}\sum_{k=1}^{n_{\mathrm{win}}} \log \widehat Z_{\bm\theta}(t_k,L),
\end{equation}
computed on a fixed set of $n_{\mathrm{win}}$ held-out windows (Table~\ref{tab:s:saem_hparams}), where $N$ is the observation dimension.

\subsubsection{SAEM updates and configurations}
For each initialization and each window length $L$, we consider three parameter-update configurations.
We write the full parameter vector as
\begin{align}
\bm\theta := (\bm\theta_{\mathrm{drift}},\,\sigma_{\mathrm{proc}},\,\sigma_{\mathrm{obs}}),\notag\\
\bm\theta_{\mathrm{drift}}:=(\bm A,\bm W,\bm h),
\end{align}
and keep the gate-noise scale fixed at $\sigma_g=0.7$ throughout.
The three configurations differ only in which blocks of $\bm\theta$ are updated:
\begin{itemize}
	\item \texttt{baseline:} no updates, the objective is evaluated at initialization. 
	\item \texttt{calib:} update only $(\sigma_{\mathrm{proc}},\sigma_{\mathrm{obs}})$ and keep $\bm\theta_{\mathrm{drift}}$ fixed.
	\item \texttt{full SAEM:} update both $\bm\theta_{\mathrm{drift}}$ and $(\sigma_{\mathrm{proc}},\sigma_{\mathrm{obs}})$.
\end{itemize}
All configurations share the same PAL-RNN switching structure with $P$ gated coordinates (Section~\ref{sec:s:rbpf_louis_details}); only the updated parameter blocks differ. The underlying AL-RNN settings and the additional SAEM/RBPF hyperparameters are listed in Table~\ref{tab:s:saem_hparams}.

Note that the training noise scales $\sigma_{\mathrm{proc}}^{\mathrm{train}}$ and $\sigma_{\mathrm{obs}}^{\mathrm{train}}$ refer to the data-generation regime used when fitting the initialization checkpoint, and are distinct from the SAEM/RBPF noise parameters $(\sigma_{\mathrm{proc}},\sigma_{\mathrm{obs}})$ optimized during fine-tuning. For numerical stability, whenever $(\sigma_{\mathrm{proc}},\sigma_{\mathrm{obs}})$ are updated we impose a small positive floor $\sigma_{\min}=10^{-4}$ in the M-step (equivalently, estimated residual variances are clamped below by $\sigma_{\min}^2$).

For each initialization we run SAEM for all window sizes
$L\in\{16,32,64,128,200\}$, yielding 60 initialization-window settings per configuration.

We optimize the windowed evidence objective via stochastic approximation expectation maximization (SAEM), using a Monte Carlo E-step based on samples from the RBPF smoothing distribution \citep{delyon1999saem}. Operationally, we implement SAEM as damped EM updates on random mini-batches of windows.

At each iteration we (i) approximate the E-step expectations under the window-wise smoothing distribution via RBPF backward sampling, (ii) compute the corresponding complete-data sufficient statistics (restricted to the parameter blocks updated in the chosen configuration), (iii) compute a ridge-regularized M-step solution to obtain a provisional update, and (iv) blend this provisional update into the running parameters.

We use a fixed blend factor $\alpha_M$ (``blend step'' in Table~\ref{tab:s:saem_hparams}) and update rule
\begin{equation}
\bm\theta^{(r+1)}\;:=\;(1-\alpha_M)\,\bm\theta^{(r)}+\alpha_M\,\tilde{\bm\theta}^{(r)},
\end{equation}
where $\tilde{\bm\theta}^{(r)}$ denotes the ridge-regularized M-step solution for the chosen parameter blocks based on the Monte Carlo E-step at iteration $r$.
Pseudo-code is given in Algorithm~\ref{alg:s:saem}.

\begin{algorithm}[t]
\caption{Particle-SAEM update on windowed evidence}
\label{alg:s:saem}
\begin{algorithmic}[1]
\STATE \textbf{Input:} sequences $\{\bm x^{(j)}_{1:T}\}_{j=1}^{n_{\mathrm{seq}}}$; window length $L$, initial parameters $\bm\theta^{(1)}$, configuration (\texttt{baseline} / \texttt{calib} / \texttt{full SAEM}), fixed gate noise $\sigma_g$.
\STATE \textbf{Hyperparameters:} SAEM iterations $R$, windows per iteration $B$, RBPF particles $N_p$, smoothing samples $S$, resampling threshold $\tau_{\mathrm{ESS}}$, ridge regularization weight $\lambda_R$,
blend factor $\alpha_M$.
\FOR{$r=1,\dots,R$}
\STATE Sample a mini-batch of $B$ windows (sequence index $j$ and start time $t_k$) from the training windows.
\FOR{each sampled window $(j,t_k)$}
\STATE Run RBPF filtering on $\bm x^{(j)}_{t_k:t_k+L}$ (conditioning on $\bm x^{(j)}_{t_k}$) and generate $S$ approximate smoothed trajectories by backward sampling.
\STATE From the smoothing samples, form Monte Carlo estimates of the complete-data sufficient statistics needed for the M-step (restricted to the parameter blocks updated by the chosen configuration).
\ENDFOR
\STATE Aggregate sufficient statistics over the mini-batch (empirical mean over windows and smoothing samples).
\STATE Compute a ridge-regularized M-step solution $\tilde{\bm\theta}^{(r)}$ using ridge parameter $\lambda_R$.
\STATE Apply the damped update $\bm\theta^{(r+1)}\leftarrow (1-\alpha_M)\bm\theta^{(r)}+\alpha_M\tilde{\bm\theta}^{(r)}$, respecting the chosen configuration.
\ENDFOR
\STATE \textbf{Return:} final parameters $\bm\theta^{(R+1)}$.
\end{algorithmic}
\end{algorithm}

\begin{table}[t]
\centering
\caption{Settings used in the SAEM fine-tuning experiments in Fig.~\ref{fig:results}c.}
\label{tab:s:saem_hparams}
{\small\setlength{\tabcolsep}{4pt}\renewcommand{\arraystretch}{1.1}
\begin{tabular}{l l}
\hline
Quantity & Value \\
\hline
\multicolumn{2}{l}{\textbf{AL-RNN initialization (ITF pretraining)}}\\
Number of models & 12 AL-RNN initializations\\
Dataset & Lorenz-63 \\
Pretraining objective & ITF next-step MSE on standardized observations (Eq.~\eqref{eq:a:tf_loss}) \\
Training trajectory length $T_{\mathrm{train}}$ & $8\cdot 10^4$ \\
Test trajectory length $T_{\mathrm{test}}$ & $2\cdot 10^4$ \\
Noise regimes & $\sigma_{\mathrm{proc}}^{\mathrm{train}}\in\{0.1,0.3,0.5\}$, $\sigma_{\mathrm{obs}}^{\mathrm{train}}\in\{0,0.1,0.3,0.5\}$ \\
Forcing interval $\tau$ & $16$ \\
Latent dimension $M$ & 30 \\
Number of gated units $P$ & $10$ \\
Batch size $B$ & 16 \\
BPTT sequence length $L_{\mathrm{BPTT}}$ & 200 \\
Epochs & 2000 \\
Batches per epoch & 50 \\
\hline
\multicolumn{2}{l}{\textbf{RBPF/Louis settings}}\\
Sequence length for curvature diagnostics $T$ & 200 \\
Particles $N_p$ & 256 \\
Smoothing samples $S$ & 8 \\
Resampling threshold $\tau_{\mathrm{ESS}}$ & 0.5 \\
Gate noise $\sigma_g$ & 0.7 \\
\hline
\multicolumn{2}{l}{\textbf{SAEM settings}}\\
SAEM iterations $R$ & 8 \\
Windows per iteration & 80 \\
Ridge $\lambda_R$ & $10^{-2}$ \\
Blend step $\alpha_M$ & 0.25 \\
Held-out windows $n_{\mathrm{win}}$ & 120 \\
Window sizes $L$ & $\{16,32,64,128,200\}$ \\
\hline
\end{tabular}}
\end{table}

\section{Metric definitions and computation details}
\label{sec:s:repro_and_metrics}
This section defines the scalar metrics reported in the main-text figures and fixes their normalizations. Unless otherwise stated, all quantities are computed on standardized observation coordinates.

\subsection{Toy-model diagnostics (Fig.~\ref{fig:results}a)}
The toy model in Fig.~\ref{fig:results}a reports two curvature/ambiguity summaries derived from Louis' identity.
Let $\Iohat$ denote the observed information estimate for the toy-model parameter block of interest, and let $\widehat{\E}[\Ic\mid \bm x]$ denote the corresponding complete-data information term (both computed for the same length-$T$ trajectory). In this toy model the gate posterior $p_t$ is available in closed form (Section~\ref{sec:s:louis_toy}), so these quantities are computed by directly evaluating Eq.~\eqref{eq:s:2switch_state_dependent} (with $\bm\theta=(a_0,a_1)$) and aggregating over time, rather than via RBPF.

\paragraph{Missing-information ratio (MIR).}
We report the missing-information ratio
\begin{align}
\mathrm{MIR}
&:= 1-\frac{\tr(\Iohat)}{\tr\big(\widehat{\E}[\Ic\mid \bm x]\big)}\notag\\
&=\frac{\tr\big(\widehat{\E}[\Ic\mid \bm x]-\Iohat\big)}{\tr\big(\widehat{\E}[\Ic\mid \bm x]\big)},
\end{align}
which takes values in $[0,1]$ and increases with posterior latent uncertainty (Section~\ref{sec:s:rbpf_louis_details}, ``Missing-information ratio (MIR)'').

\paragraph{Observed-curvature proxy.}
As a scalar curvature proxy we also report $\log_{10}\tr(\Iohat)$. In the toy-model experiment, this quantity is used only for relative comparison across gate-noise settings $\sigma_g$ at fixed trajectory length $T$.

\subsection{Curvature gap / sharpness mismatch (Fig.\ref{fig:results}b)}
For a trained AL-RNN checkpoint, with (i) the ITF-aligned curvature proxy $\bm{\mathcal{I}}_{\mathrm{ITF}}$ (Section~\ref{sec:s:tf_fim}) and (ii) an RBPF/Louis observed information estimate $\Iohat$ computed on a trajectory of length $T$ (Section~\ref{sec:s:rbpf_louis_details}), we define the per-step observed information trace as
\begin{equation}
\tr(\Iostephat):=\frac{\tr(\Iohat)}{T}.
\end{equation}
The curvature-gap statistic reported in Fig.~\ref{fig:results}b is
\begin{equation}
\label{eq:s:gQ_def}
g_{\bm Q} \;:=\; \log_{10}\left(\frac{\tr(\bm{\mathcal{I}}_{\mathrm{ITF}})}{\tr(\Iostephat)}\right).
\end{equation}
Both numerator and denominator are interpreted on a per-step scale (in our implementation $\bm{\mathcal{I}}_{\mathrm{ITF}}$ is stored with the per-step normalization from Eq.~\eqref{eq:a:tf_loss}).

\subsection{Windowed evidence (held-out; RBPF estimate) (Fig.~\ref{fig:results}c)}
For a fixed window length $L$, the windowed conditional log-evidence is approximated by the RBPF normalizing-constant estimate $\log\widehat Z_{\bm\theta}(t,L)$ (Section~\ref{sec:s:saem_details}).
The quantity shown on the y-axis in Fig.~\ref{fig:results}c is the held-out windowed evidence score
$\mathrm{Evidence}(\bm\theta)$ (Eq.~\eqref{eq:a2:evidence}),
defined as the mean of $\log\widehat Z$ over a fixed set of held-out windows, normalized per step and per observed dimension by $L\,N$ (Section~\ref{sec:s:saem_details}).
We emphasize that this is a \emph{windowed} and \emph{conditional} evidence estimate, not the marginal likelihood of an entire long free-running trajectory.

\subsection{Dynamical QoI metrics (Fig.~\ref{fig:results}c)}
For the QoI evaluations in Fig.~\ref{fig:results}c we use long free-running rollouts of length $T=10^4$ time steps (discarding a fixed burn-in of 1000 steps), and compute both $D_{\mathrm{stsp}}$ and Lyapunov-based quantities on the resulting generated and reference trajectories.
\paragraph{State-space divergence $D_{\mathrm{stsp}}$.}
Given generated and reference trajectories $\{\bm{x}_t^{\mathrm{gen}}\}_{t=1}^T$ and $\{\bm{x}_t^{\mathrm{true}}\}_{t=1}^T$, we construct a joint histogram over a fixed bounding box defined by the minimum/maximum of the reference trajectory, using $n_{\mathrm{bins}}$ bins per dimension ($n_{\mathrm{bins}}=30$ throughout this work).
Let $\hat p$ and $\hat q$ be the resulting discrete distributions after Laplace smoothing with $\alpha_{\mathrm{smooth}}=10^{-5}$.
Then
\begin{equation}
D_{\mathrm{stsp}} := \mathrm{KL}(\hat p\,\|\,\hat q) = \sum_{b} \hat p(b)\log\frac{\hat p(b)}{\hat q(b)}.
\end{equation}
In the SAEM rollouts, the model is simulated in the standardized coordinates used by the objective, and we then apply the inverse standardization map before computing $D_{\mathrm{stsp}}$ so that the histogram comparison is carried out in the physical (raw) Lorenz coordinates.

\paragraph{Lyapunov exponent and error.}
We evaluate the largest LE $\lambda_{\mathrm{max}}$ of the learned discrete-time dynamics using a standard method based on QR re-orthonormalization \citep{benettin1980lyapunov}, applied to a sequence of closed-form one-step model Jacobians along a rollout. Rollouts are initialized from the first observation $\bm{x}_1\in\mathbb{R}^N$ using the learned linear embedding $\bm E$ (Eq.~\eqref{eq:alrnn_init_embed}).

Along the rollout, let $\bm c_t\in\{0,1\}^P$ denote the induced switching code at time $t$ (Eq.~\eqref{eq:alrnnswitching}) and let $\bm{D}(\bm{c}_t)$ be the corresponding diagonal gate matrix. The local Jacobian $\bm{J}_t$ of the one-step map is given by Eq.~\eqref{eq:a1:stepjacobian}.

In the SAEM rollout evaluation we use hard-gated deterministic rollouts for the results shown in Fig.~\ref{fig:results}c, i.e. we evolve the state with the deterministic AL-RNN transition and set the gate bits by the ReLU sign pattern at each time step.
For additional analyses (Section~\ref{sec:s:optional}), we also consider a stochastic probit-gated rollout that samples Bernoulli gates with probabilities $\Phi(z_{t,j}/\sigma_g)$ on the gated coordinates.

Given Jacobians $\{\bm{J}_t\}_{t=1}^T$ of the model, the discrete-time Lyapunov spectrum is estimated by QR re-orthonormalization \citep{benettin1980lyapunov}:
\begin{equation}
\widehat{\lambda}_k^{\mathrm{disc}} = \frac{1}{T}\sum_{t=1}^T \log\left|R_{t,kk}\right|,\quad \text{where } \bm{J}_t\bm{Q}_{t-1}=\bm{Q}_t\bm{R}_t.
\end{equation}
Continuous-time exponents are then obtained through dividing by $\Delta t$.
We report $\widehat{\lambda}_1 := \widehat{\lambda}_{\mathrm{max}}$ as the largest LE of the model after discarding an initial 1000 time steps of the rollout.

For the Lorenz-63 reference value we use the literature value $\lambda_{\mathrm{max}}^{\mathrm{ref}}=0.9056$ \citep{sparrow1982lorenz}. For $\sigma_{\mathrm{proc}} > 0$ we interpret this as a reference for the underlying noiseless drift dynamics; a noise-matched comparison would require a random-DS/SDE notion of LEs and a separate stochastic reference estimate, which we do not compute here. We report the signed Lyapunov error $\widehat{\lambda}_1-\lambda_{\mathrm{max}}^{\mathrm{ref}}$.

\subsection{Switching code diagnostics}
In the RBPF/Louis recomputation we report $H_c$, the time-averaged RBPF filtering switching code entropy (in bits) defined in Section~\ref{sec:s:rbpf_louis_details}.

\section{Additional analyses}
\label{sec:s:optional}
This section collects additional analyses not included in the main text. We (i) probe robustness of the entropy--gap association under stratifications, (ii) sanity-check RBPF computations with time-resolved diagnostics, (iii) provide qualitative visualizations that help interpret the quantitative QoI metrics and (iv) explore additional curvature diagnostics that complement the trace-based curvature-gap statistic in Fig.~\ref{fig:results}b. We also briefly discuss the computational cost of the RBPF/Louis computations and the SAEM fine-tuning procedure.

\subsection[Fixed-Q stratification of entropy--gap association]{Fixed-$Q$ stratification of entropy--gap association}
\label{sec:s:fixedQ_strat}

To check whether the entropy--gap association reported in the main paper (Fig.~\ref{fig:results}b; $g_{\bm Q}$ versus $H_c$ at fixed $\bm Q$) depends strongly on the fixed process-noise level $\sigma_{\mathrm{proc}}$ (i.e., fixed-$Q$ strata), we re-analyzed the fixed-$Q$ stochastic checkpoints ($\sigma_{\mathrm{obs}}>0$) separately for each $\sigma_{\mathrm{proc}}\in\{0.1,0.3,0.5\}$.
Each stratum contained $n=60$ checkpoints. Across strata, the rank association between RBPF filtering switching code entropy $H_c$ and curvature gap $g_{\bm Q}$ remained positive. However, after rank-residualizing both variables against $\sigma_{\mathrm{obs}}$, the partial correlations are near zero or negative. Thus, the positive within-stratum rank association is largely explained by variation in observation-noise level inside each process-noise stratum, rather than by a strong residual model-to-model entropy effect at fixed noise settings. Table~\ref{tab:s:fixedQ_strat} summarizes the per-stratum results.

Note that $H_c$ is computed from the RBPF \emph{filtering} distribution (Section~\ref{sec:s:rbpf_louis_details}), and therefore reflects an online ambiguity diagnostic (ITF-aligned). Smoothing-based entropies will typically be smaller.

\begin{table}[t]
\centering
\caption{Fixed-$Q$ stratification of the entropy--gap association (stochastic checkpoints with $\sigma_{\mathrm{obs}}>0$). Within each fixed-$\sigma_{\mathrm{proc}}$ stratum ($n=60$ checkpoints), we report the Spearman rank correlation between RBPF filtering code entropy $H_c$ and curvature gap $g_{\bm Q}$ (Fig.~\ref{fig:results}b). ``Partial $r$'' controls for $\sigma_{\mathrm{obs}}$ via rank-residualization. $95\%$ CIs are bootstrap percentiles (2000 resamples).}
\label{tab:s:fixedQ_strat}
{\small\setlength{\tabcolsep}{3pt}\renewcommand{\arraystretch}{1.1}
\begin{tabular}{c c c c c c c c c}
\hline
$\sigma_{\mathrm{proc}}$ & $n$ & $H_{\min}$ & $H_{\max}$ & $H$ span & Spearman $r$ & Spearman $p$ & Partial $r$ & Spearman $95\%$ CI \\
\hline
0.1 & 60 & 2.173 & 4.352 & 2.180 & 0.630 & $2.0\times 10^{-7}$ & 0.016 & $[0.426,\ 0.773]$ \\
0.3 & 60 & 4.314 & 5.929 & 1.615 & 0.628 & $7.8\times 10^{-8}$ & -0.039 & $[0.438,\ 0.751]$ \\
0.5 & 60 & 5.488 & 6.443 & 0.955 & 0.585 & $9.3\times 10^{-7}$ & -0.310 & $[0.410,\ 0.717]$ \\
\hline
\end{tabular}
}
\end{table}

\subsection{RBPF diagnostics}
The main text focuses on scalar summaries of switching ambiguity and curvature (e.g., $H_c$ and trace-based curvature proxies).
As a basic sanity check for the RBPF computations underlying these summaries (Section~\ref{sec:s:rbpf_louis_details}), we optionally inspect time series diagnostics that probe particle degeneracy and switching uncertainty.

\paragraph{Effective sample size (ESS).}
We monitor the effective sample size $\mathrm{ESS}_t$ defined in Eq.~\eqref{eq:rbpf_ess}. Resampling is triggered when $\mathrm{ESS}_t < \tau_{\mathrm{ESS}}N_p$, where the threshold $\tau_{\mathrm{ESS}}$ is reported in Table~\ref{tab:s:saem_hparams}.

\paragraph{Resampling events.}
We record the time indices at which resampling occurs. Frequent resampling is expected in highly informative (low-noise) regimes but can also indicate proposal mismatch. Conversely, persistently low ESS without resampling indicates degeneracy.

\paragraph{Time-resolved switching-code entropy.}
At each transition time, we compute the filtering-distribution switching uncertainty via the full-code entropy $H_{t,\mathrm{code}}$ (Section~\ref{sec:s:rbpf_louis_details}). We report this trace in bits, i.e. we plot
\begin{equation}
H^{(\mathrm{bits})}_{t,\mathrm{code}} := H_{t,\mathrm{code}}/\log 2.
\end{equation}
This provides a time-resolved view of ambiguity that complements the time-averaged summary $H_c$.

Figure~\ref{fig:s:rbpf_diagnostics} illustrates these optional diagnostics on a representative SAEM evaluation window (window length $L=32$). All RBPF settings for this run match Table~\ref{tab:s:saem_hparams}. This visualization serves to exemplify that the reported entropy and curvature summaries are not dominated by particle collapse artifacts.

\begin{figure}[t]
\centering
\includegraphics[width=1.0\textwidth]{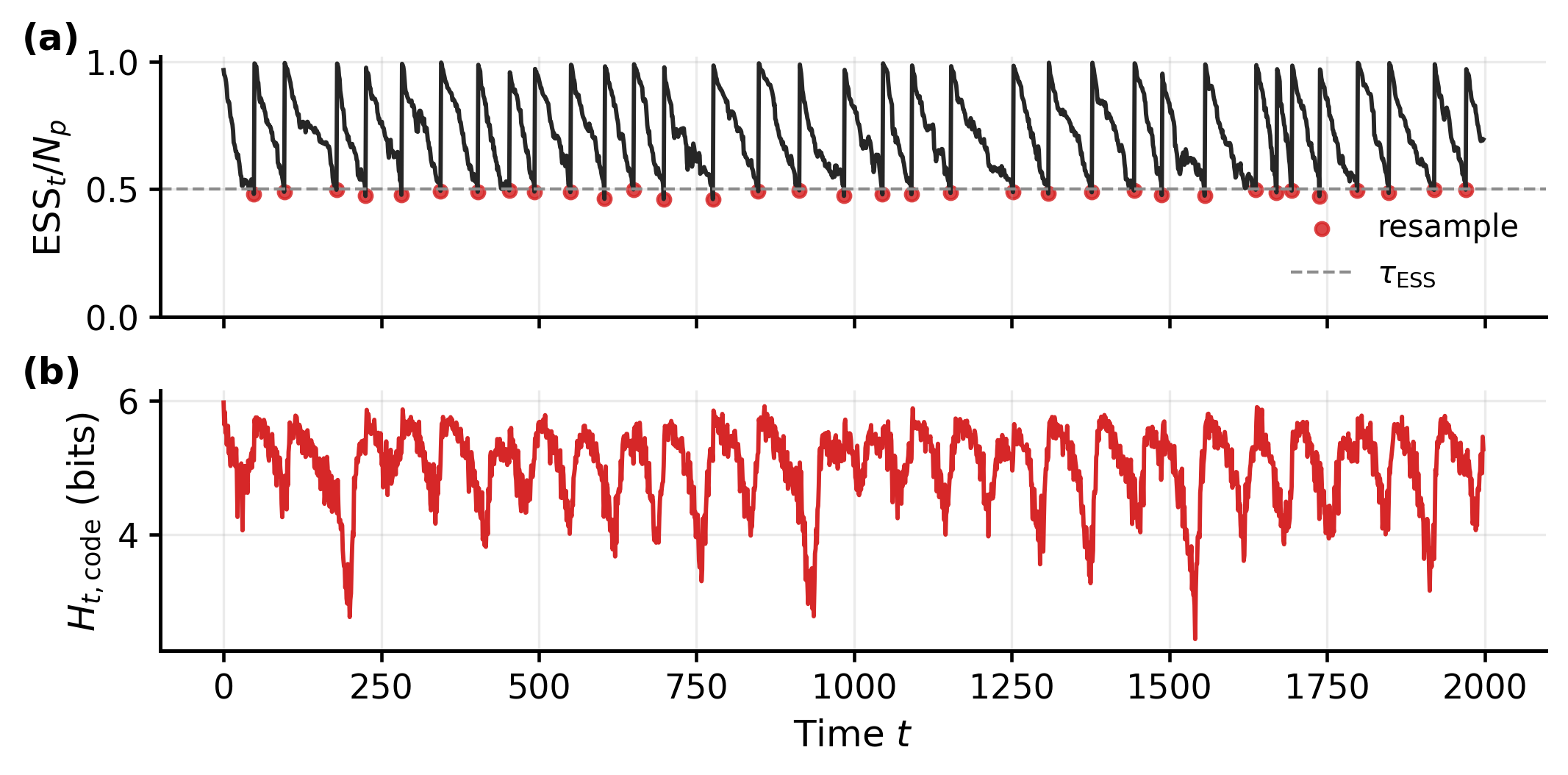}
\caption{RBPF diagnostics on a representative SAEM evaluation window. Panels: \textbf{(a)} normalized ESS ($\mathrm{ESS}_t/N_p$) with resampling events (triggered when $\mathrm{ESS}_t<\tau_{\mathrm{ESS}}N_p$); \textbf{(b)} full-code filtering entropy in bits, $H^{(\mathrm{bits})}_{t,\mathrm{code}}:=H_{t,\mathrm{code}}/\log 2$. These diagnostics help confirm that particle degeneracy is controlled under the chosen RBPF budget and that the switching-uncertainty summaries (e.g., $H_c$) reflect switching ambiguity rather than particle collapse.}
\label{fig:s:rbpf_diagnostics}
\end{figure}

\subsection{Matrix-aware mismatch diagnostics (anisotropy and eigenspace misalignment)}
\label{sec:s:matrix_mismatch}
The main text focuses on scalar (trace-based) curvature summaries because they are robust and easy to interpret. Figure~\ref{fig:s:matrix_mismatch} complements this view by probing whether the mismatch between the ITF-aligned curvature proxy $\bm{\mathcal I}_{\mathrm{ITF}}$ (Section~\ref{sec:s:tf_fim}) and the ambiguity-aware Louis observed information estimate $\Iohat$ (Section~\ref{sec:s:rbpf_louis_details}) is primarily a uniform rescaling, or whether it is anisotropic and/or accompanied by eigenspace misalignment.

Throughout, we compare matrices on the same segment and on a \emph{per-step} scale. Let
\begin{equation}
\Iostephat := \Iohat/T.
\end{equation}
Recall that $\bm{\mathcal I}_{\mathrm{ITF}}$ already includes the per-step normalization from the ITF loss (Section~\ref{sec:s:tf_fim}). To stabilize the Monte Carlo Louis estimator at finite $N_p,S$, we apply ridge stabilization and symmetrize the respective matrices before computing the following diagnostics. Let $p$ be the parameter dimension of the block under study (here, the drift parameters $(\bm A,\bm W,\bm h)$ with $p=M^2+2M$) and define
\begin{equation}
\mu := \varepsilon\,\frac{\tr(\Iostephat)+\tr(\bm{\mathcal I}_{\mathrm{ITF}})}{2p},\qquad \varepsilon=10^{-6}.
\end{equation}
We then set
\begin{align}
\Iostephat^{(\mu)}&:=\tfrac12(\Iostephat+\Iostephat^{\mathsf T})+\mu\bm I,\notag\\
\bm{\mathcal I}_{\mathrm{ITF}}^{(\mu)}&:=\tfrac12\big(\bm{\mathcal I}_{\mathrm{ITF}}+\bm{\mathcal I}_{\mathrm{ITF}}^{\mathsf T}\big)+\mu\bm I.
\end{align}

\paragraph{Log-det gap $\Delta_{\log\det}$.}
We report the ridge-stabilized log-determinant difference in decades,
\begin{equation}
\Delta_{\log\det}
:=
\frac{\log\det\big(\bm{\mathcal I}_{\mathrm{ITF}}^{(\mu)}\big)-\log\det\big(\Iostephat^{(\mu)}\big)}{\log 10}.
\end{equation}
This is a global (volume) mismatch summary that is sensitive to anisotropy, not only to trace scaling.

\paragraph{Direction-wise curvature ratios via generalized eigenvalues $q_{0.9}^{\gamma}$.}
To characterize curvature ratios in a direction-dependent way, we consider the generalized eigenproblem
\begin{equation}
\label{eq:gen_eig}
\bm{\mathcal I}_{\mathrm{ITF}}^{(\mu)}\bm u_i = \gamma_i\,\Iostephat^{(\mu)}\bm u_i,
\qquad i=1,\dots,p,
\end{equation}
equivalently $\{\gamma_i\}_{i=1}^p=\mathrm{eig}((\Iostephat^{(\mu)})^{-1}\bm{\mathcal I}_{\mathrm{ITF}}^{(\mu)})$.
Here $\gamma_i$ compares the ITF curvature to the ambiguity-aware observed information along the generalized-eigenvector direction $\bm u_i$; $\gamma_i>1$ indicates a direction that is sharper under ITF than under the marginal likelihood geometry. We summarize the spectrum by quantiles of $\log_{10}\gamma_i$ and define the shorthand
\begin{equation}
q_{\alpha}^{\gamma}
:=
\mathrm{Quantile}_{\alpha}\!\big(\{\log_{10}\gamma_i\}_{i=1}^p\big),\qquad \alpha\in(0,1).
\end{equation}
In Fig.~\ref{fig:s:matrix_mismatch}(b) we report the 90th percentile summary $q_{0.9}^{\gamma}$, which captures the upper end of the curvature ratio distribution and is sensitive to the presence of directions that are much sharper under ITF than under the marginal likelihood geometry. Note that $q_{0.9}^{\gamma}>0$ indicates that at least 10\% of the generalized eigenvalues $\gamma_i$ exceed 1, i.e. there are nontrivial directions along which the ITF curvature is sharper than the ambiguity-aware observed information. Conversely, $q_{0.9}^{\gamma}<0$ indicates that at least 90\% of the $\gamma_i$ are below 1, i.e. most directions are sharper under the marginal likelihood geometry than under ITF.

\paragraph{Top-$k$ sensitive-subspace overlap $\mathrm{ov}_{k}$.}
Finally, we compare the leading eigenspaces of $\bm{\mathcal I}_{\mathrm{ITF}}^{(\mu)}$ and $\Iostephat^{(\mu)}$. Let $\bm U_{\mathrm{ITF}},\bm U_{\mathrm{obs}}\in\mathbb{R}^{p\times k}$ collect the top-$k$ eigenvectors of $\bm{\mathcal I}_{\mathrm{ITF}}^{(\mu)}$ and $\Iostephat^{(\mu)}$, respectively. We define
\begin{equation}
\mathrm{ov}_k := \frac{1}{k}\,\|\bm U_{\mathrm{ITF}}^{\mathsf T}\bm U_{\mathrm{obs}}\|_F^2\in[0,1],
\end{equation}
the mean squared cosine of principal angles between the two $k$-dimensional subspaces (since the eigenvectors can be chosen orthonormal, the singular values of $\bm U_{\mathrm{ITF}}^{\mathsf T}\bm U_{\mathrm{obs}}$ equal $\cos\theta_i$ and $\|\cdot\|_F^2$ sums $\cos^2\theta_i$).
In Fig.~\ref{fig:s:matrix_mismatch}(c) we use $k=50$ and denote the reported overlap by $\mathrm{ov}_{50}$.

\paragraph{Interpretation (Fig.~\ref{fig:s:matrix_mismatch}).}

Across the same checkpoints and fixed-$\sigma_{\mathrm{proc}}$ segments used in Fig.~\ref{fig:results}b, the matrix-aware summaries show that $\Delta_{\log\det}$ is negative and $q_{0.9}^{\gamma}<0$, meaning that at least $90\%$ of generalized-eigenvalue directions have $\gamma_i<1$. Thus, most directions are not sharper under ITF. The mismatch is instead anisotropic and concentrated: trace-level gaps can coexist with broad direction-wise shrinkage, while the leading-subspace overlap remains low and decreases with switching entropy. We interpret these diagnostics as evidence for subspace-dependent geometry mismatch instead of broad direction-wise overconfidence of ITF.

\begin{figure}[t]
\centering
\includegraphics[width=\linewidth]{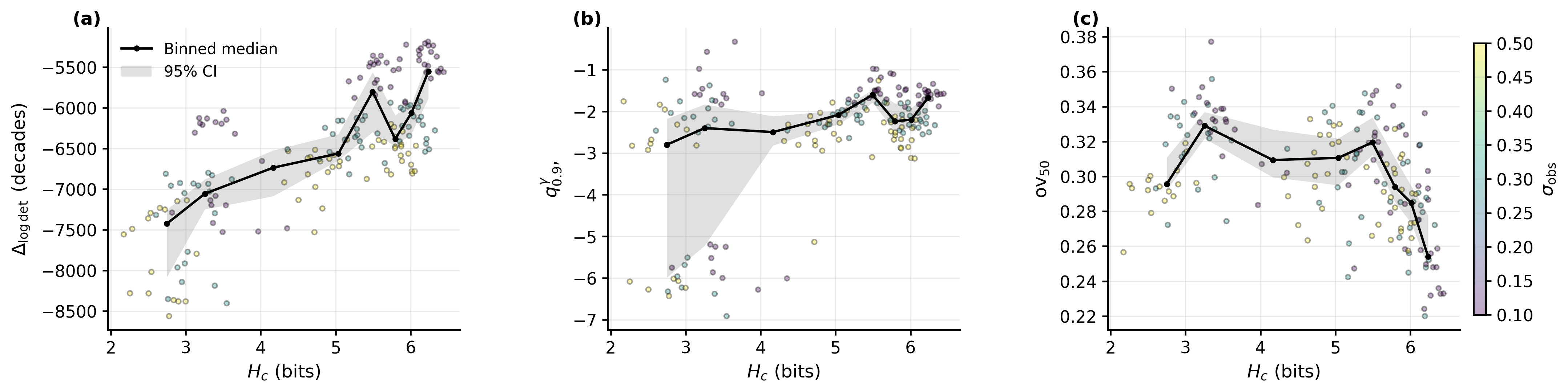}
\caption{Matrix-aware mismatch diagnostics comparing ITF curvature $\bm{\mathcal I}_{\mathrm{ITF}}$ to ambiguity-aware observed information $\Iohat$ on the same segments, stratified by fixed process noise $\sigma_{\mathrm{proc}}$. Points are colored by observation-noise level $\sigma_{\mathrm{obs}}$. \textbf{(a)} Ridge-stabilized log-det gap in decades, $\Delta_{\log\det}$. \textbf{(b)} Upper-quantile direction-wise ratio, $q_{0.9}^{\gamma}$, where $\gamma_i$ are generalized eigenvalues of $(\Iostephat^{(\mu)})^{-1}\bm{\mathcal I}_{\mathrm{ITF}}^{(\mu)}$ (Eq.~\eqref{eq:gen_eig}). \textbf{(c)} Leading-subspace overlap, $\mathrm{ov}_{50}$, between the top-$50$ eigenspaces of $\bm{\mathcal I}_{\mathrm{ITF}}^{(\mu)}$ and $\Iostephat^{(\mu)}$.}
\label{fig:s:matrix_mismatch}
\end{figure}

\begin{figure}[t]
\centering
\includegraphics[width=1.0\textwidth]{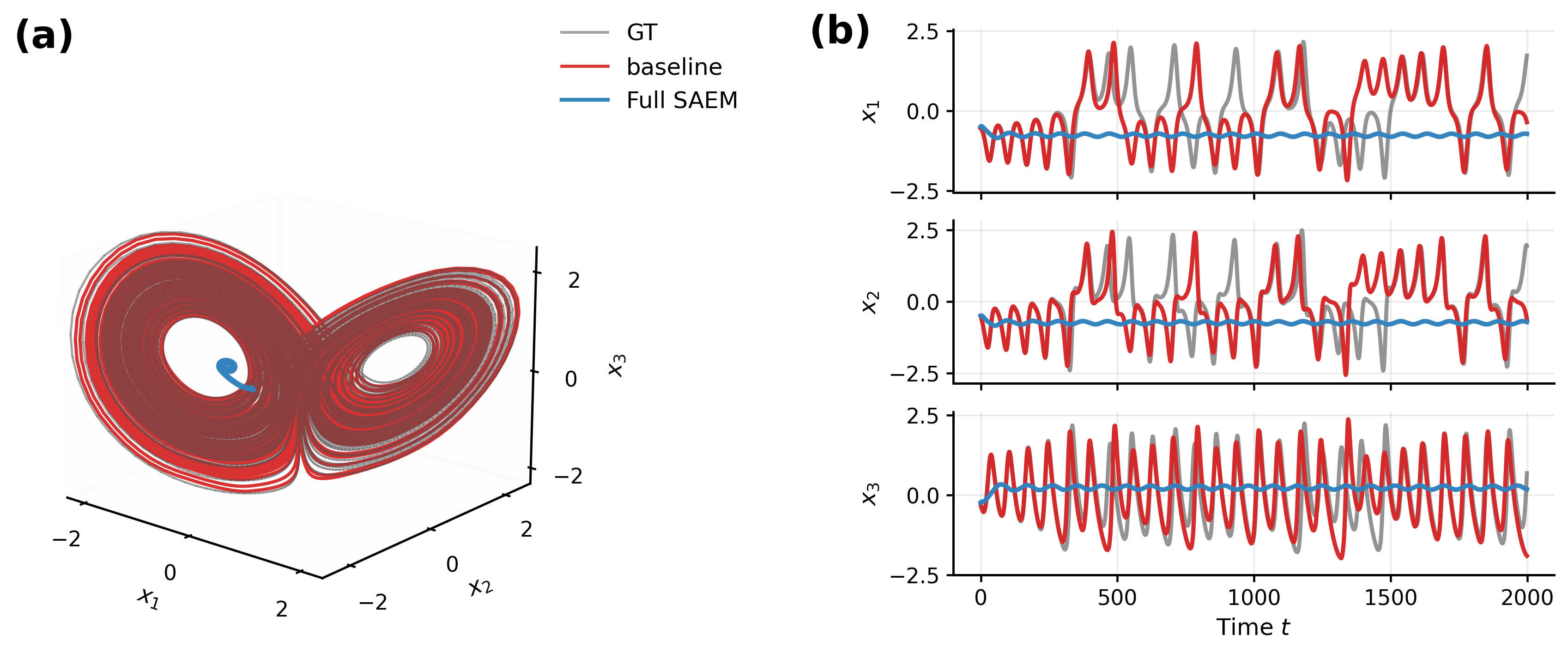}
\caption{Qualitative attractor reconstruction visualization for a representative AL-RNN initialization checkpoint and its SAEM-fine-tuned PAL-RNN counterpart (``\texttt{full SAEM}'' in Fig.~\ref{fig:results}c). Rollouts are length $T=5\cdot10^4$ time steps and initialized at the first ground-truth state. \textbf{(a)} Attractor reconstruction. Baseline closely matches the ground truth (GT) attractor geometry, while the deterministic rollout of the SAEM-fine-tuned model collapses to a strongly contracted periodic orbit. \textbf{(b)} Zoomed-in view of the trajectories, with GT, baseline, and SAEM-fine-tuned rollouts initialized at the same point to highlight sensitive dependence on initial conditions: while the baseline trajectory (necessarily) diverges in phase from the GT reference, it still captures the overall attractor geometry.}
\label{fig:s:attractor_overlays}
\end{figure}
\subsection{Qualitative attractor reconstruction visualization}
As a qualitative complement to the scalar QoI metrics in Section~\ref{sec:s:repro_and_metrics}, we visualize attractor reconstruction for a representative AL-RNN initialization checkpoint and its SAEM-fine-tuned PAL-RNN counterpart (``\texttt{full SAEM}'' in Fig.~\ref{fig:results}c). Figure~\ref{fig:s:attractor_overlays} complements the scalar QoI metrics in Fig.~\ref{fig:results}c and illustrates that improving windowed evidence need not preserve the qualitative attractor geometry.

\subsection{Computational complexity considerations}
\label{sec:s:complexity}
Let $T$ be the sequence length, $M$ the latent dimension, $N$ the observation dimension, and $p$ the parameter dimension of the drift block (here $p=M^2+2M$ for $\bm A,\bm W,\bm h$).
Computing the ITF Fisher proxy in closed form (Section~\ref{sec:s:tf_fim}) requires propagating sensitivities $\bm S_t\in\mathbb{R}^{M\times p}$ and accumulating a $p\times p$ outer product. 
In a straightforward implementation, this costs on the order of
\begin{equation}
\mathcal O\big(T\,M\,p + T\,N\,p + T\,p^2\big) = \mathcal O(T\,M^4)\quad\text{(for fixed $N$)}
\end{equation}
time, with memory $\mathcal O(Mp + p^2)=\mathcal O(M^4)$ if the full matrix is stored. Thus the cost is dominated by the number of latent dimensions $M$ of the model.

Considering the estimation of $\Iohat$ via RBPF/Louis (Section~\ref{sec:s:rbpf_louis_details}), with $N_p$ particles and $S$ approximate smoothing trajectories, the cost scales roughly as
\begin{equation}
\mathcal O\big((N_p+S)\,T\,\mathrm{KF}(M)\big),
\end{equation}
where $\mathrm{KF}(M)$ denotes the per-step cost of the conditional linear-Gaussian update ($\mathcal O(M^3)$ in the worst case, i.e. for dense covariances).

\section{Limitations and future work}
\label{sec:s:limitations}
All curvature quantities used here ($\bm{\mathcal I}_{\mathrm{ITF}}$ and $\Io$) are inherently \emph{local} summaries around a parameter value. Interpreting them as posterior precision (e.g., via Laplace/Bernstein--von Mises-type arguments) requires additional regularity and, at least locally, approximately Gaussian behavior of the objective. In non-convex switching models, the \emph{global} geometry can differ substantially due to multimodality and flat directions, so curvature is best read as a diagnostic of local sensitivity rather than a complete description of posterior shape.

Further, our observed-information and fine-tuning results rely on (i) Monte Carlo RBPF/Louis approximations for the PAL-RNN, (ii) a windowed conditional evidence objective, and (iii) experiments on Lorenz--63 initialized from ITF-trained AL-RNN checkpoints. These choices bound the scope of the conclusions.

First, the evidence objective in Fig.~\ref{fig:results}c is computed on fixed-length windows as $\log p_{\bm\theta}(\bm x_{t+1:t+L}\mid \bm x_t)$ (Section~\ref{sec:s:saem_details}), rather than as a full-sequence marginal likelihood for a single long free-running trajectory. Consequently, improving windowed evidence is not guaranteed to improve long-horizon dynamical QoIs, and can favor parameter changes that improve short-window predictability while altering invariant structure. This misalignment is visible empirically when drift parameters are updated (\texttt{full SAEM}). We view windowing as a pragmatic stabilization of likelihood estimation for chaotic, long-horizon sequences, not as a substitute for QoI-aligned training.

Second, the RBPF/Louis observed-information estimates are Monte Carlo approximations whose accuracy depends on particle/smoothing budgets, resampling behavior, and the chosen PAL-RNN augmentation (Section~\ref{sec:s:rbpf_louis_details}). In more nonlinear regimes or at longer horizons, these estimators can suffer from particle degeneracy or high variance; stability under increased budgets and alternative proposals is therefore an important practical consideration (Fig.~\ref{fig:s:rbpf_diagnostics}).

Third, the PAL-RNN uses simplified noise structure (isotropic Gaussian process and observation noise, and fixed gate-noise scale $\sigma_g$). Extending the analysis to richer noise models (e.g., anisotropic or state-dependent noise), to learning or calibrating $\sigma_g$, and to partial-observation settings would improve realism and help identify which aspects of the geometry mismatch persist.

Fourth, the dynamical QoIs reported for particle-SAEM fine-tuning are evaluated on deterministic, hard-gated rollouts (noise set to zero after thresholding the learned gate probabilities). This isolates drift dynamics for comparability, but it does not evaluate invariants of the full stochastic PAL-RNN implied by the fitted noise parameters. A natural next step is to estimate QoIs directly from sampled free-running PAL-RNN trajectories (and their induced stationary distributions), and to compare drift-only invariants to their distributional/stochastic counterparts.

On the inference/geometry side, it would be valuable to develop scalable ambiguity-aware curvature summaries beyond trace proxies (e.g., leading-eigenspace estimates or structured approximations) and to connect them more directly to downstream goals such as uncertainty quantification/calibration, robust fine-tuning, and optimal experimental design. Extending the empirical comparison beyond strict ITF to GTF (e.g., sweeping forcing strength $\alpha$ in Eq.~(\ref{eq:a0:gtf_mix})) would further clarify how intervention policies interpolate between stable training geometry and ambiguity-aware likelihood geometry.

Finally, the ITF-aligned Fisher proxy $\bm{\mathcal I}_{\mathrm{ITF}}$ exploits AL-RNN structure to yield a closed-form, training-aligned curvature summary under a specific intervention policy. Extending such intervention-aligned curvature proxies to other surrogate classes and downstream applications is a promising direction.

\end{document}